\documentclass[11pt]{article}

\usepackage{fullpage,times}

\usepackage{amsthm,amsfonts,amsmath,amssymb,epsfig,color,float,graphicx,verbatim}
\usepackage{algorithm,algorithmic}

\usepackage{hyperref}
\hypersetup{
	colorlinks   = true, 
	urlcolor     = blue, 
	linkcolor    = blue, 
	citecolor   = red 
}

\newtheorem{theorem}{Theorem}
\newtheorem{proposition}{Proposition}
\newtheorem{lemma}{Lemma}
\newtheorem{corollary}{Corollary}
\newtheorem{definition}{Definition}
\newtheorem{remark}{Remark}

\newcommand{\reals}{\mathbb{R}}
\newcommand{\E}{\mathbb{E}}

\newcommand{\bx}{\mathbf{x}}
\newcommand{\bw}{\mathbf{w}}
\newcommand{\bg}{\mathbf{g}}
\newcommand{\bb}{\mathbf{b}}
\newcommand{\bu}{\mathbf{u}}
\newcommand{\bv}{\mathbf{v}}

\newcommand{\bn}{\mathbf{n}}

\newcommand{\bs}{\mathbf{s}}

\newcommand{\Ocal}{\mathcal{O}}

\newcommand{\Ucal}{\mathcal{U}}

\newcommand{\Rcal}{\mathcal{R}}

\newcommand{\Scal}{\mathcal{S}}

\newcommand{\Vcal}{\mathcal{V}}
\newcommand{\Wcal}{\mathcal{W}}

\newcommand{\norm}[1]{\left\|#1\right\|}
\newcommand{\inner}[1]{\left\langle#1\right\rangle}

\newcommand{\secref}[1]{Section~\ref{#1}}
\newcommand{\subsecref}[1]{Subsection~\ref{#1}}

\renewcommand{\eqref}[1]{Eq.~(\ref{#1})}
\newcommand{\lemref}[1]{Lemma~\ref{#1}}
\newcommand{\corref}[1]{Corollary~\ref{#1}}
\newcommand{\thmref}[1]{Thm.~\ref{#1}}

\title{Without-Replacement Sampling for Stochastic Gradient Methods: Convergence Results and Application to Distributed Optimization}
\author{Ohad Shamir\\Weizmann Institute of Science\\\texttt{ohad.shamir@weizmann.ac.il}}

\date{}

\begin{document}

\maketitle

\begin{abstract}
	Stochastic gradient methods for machine learning and optimization problems are usually analyzed assuming data points are sampled \emph{with} replacement. In practice, however, sampling \emph{without} replacement is very common, easier to implement in many cases, and often performs better. In this paper, we provide competitive convergence guarantees for without-replacement sampling, under various scenarios, for three types of algorithms: Any algorithm with online regret guarantees, stochastic gradient descent, and SVRG. A useful application of our SVRG analysis is a nearly-optimal algorithm for regularized least squares in a distributed setting, in terms of both communication complexity and runtime complexity, when the data is randomly partitioned and the condition number can be as large as the data size per machine (up to logarithmic factors). Our proof techniques combine ideas from stochastic optimization, adversarial online learning, and transductive learning theory, and can potentially be applied to other stochastic optimization and learning problems.
\end{abstract}

\section{Introduction}

Many canonical machine learning problems boil down to solving a convex empirical risk minimization problem of the form
\begin{equation}\label{eq:obj}
\min_{\bw\in \Wcal} F(\bw) = \frac{1}{m}\sum_{i=1}^{m}f_i(\bw),
\end{equation}
where each individual function $f_i(\cdot)$ is convex (e.g. the loss on a given example in the training data), and the set $\Wcal\subseteq \reals^d$ is convex. In large-scale applications, where both $m,d$ can be huge, a very popular approach is to employ stochastic gradient methods. Generally speaking, these methods maintain some iterate $\bw_t\in \Wcal$, and at each iteration, sample an individual function $f_i(\cdot)$, and perform some update to $\bw_t$ based on $\nabla f_i(\bw_t)$. Since the update is with respect to a single function, this update is usually computationally cheap. Moreover, when the sampling is done independently and uniformly at random, $\nabla f_i(\bw_t)$ is an unbiased estimator of the true gradient $\nabla F(\bw_t)$, which allows for good convergence guarantees after a reasonable number of iterations (see for instance \cite{shalev2014understanding,bertsekas1999nonlinear,bubeck2015convex,rakhlin2011making}).

However, in practical implementations of such algorithms, it is actually quite common to use \emph{without-replacement} sampling, or equivalently, pass sequentially over a random shuffling of the functions $f_i$. Intuitively, this forces the algorithm to process more equally all data points, and often leads to better empirical performance. Moreover, without-replacement sampling is often easier and faster to implement, as it requires sequential data access, as opposed to the random data access required by with-replacement sampling (see for instance \cite{bottou2009curiously,bottou2012stochastic,feng2012towards,recht2012beneath,gurbuzbalaban2015random}).

\subsection{What is Known so Far?}

Unfortunately, without-replacement sampling is not covered well by current theory. The challenge is that unlike with-replacement sampling, the functions processed at every iteration are not statistically independent,  and their correlations are difficult to analyze. Since this lack of theory is the main motivation for our paper, we describe the existing known results in some detail, before moving to our contributions.

To begin with, there exist classic convergence results which hold deterministically for every order in which the individual functions are processed, and in particular when we process them by sampling without replacement (e.g. \cite{nedic2001convergence}). However, these can be exponentially worse than those obtained using random without-replacement sampling, and this gap is inevitable (see for instance \cite{recht2012beneath}). 

More recently, Recht and R\'{e} \cite{recht2012beneath} studied this problem, attempting to show that at least for least squares optimization, without-replacement sampling is always better (or at least not substantially worse) than with-replacement sampling on a given dataset. They showed this reduces to a fundamental conjecture about arithmetic-mean inequalities for matrices, and provided partial results in that direction, such as when the individual functions themselves are assumed to be generated i.i.d. from some distribution. However, the general question remains open. 

In a recent breakthrough, G\"{u}rb\"{u}zbalaban et al. \cite{gurbuzbalaban2015random} provided a new analysis of gradient descent algorithms for solving \eqref{eq:obj} based on random reshuffling: Each epoch, the algorithm draws a new permutation on $\{1,\ldots,m\}$ uniformly at random, and processes the individual functions in that order. Under smoothness and strong convexity assumptions, the authors obtain convergence guarantees of essentially $\Ocal(1/k^2)$ after $k$ epochs, vs. $\Ocal(1/k)$ using with-replacement sampling (with the $\Ocal(\cdot)$ notation including certain dependencies on the problem parameters and data size). Thus, without-replacement sampling is shown to be strictly better than with-replacement sampling, after sufficiently many passes over the data. However, this leaves open the question of why without-replacement sampling works well after a few -- or even just one -- passes over the data. Indeed, this is often the regime at which stochastic gradient methods are most useful, do not require repeated data reshuffling, and their good convergence properties are well-understood in the with-replacement case.

\subsection{Our Results}

In this paper, we provide convergence guarantees for stochastic gradient methods, under several scenarios, in the natural regime where the number of passes over the data is small, and in particular that no data reshuffling is necessary. We emphasize that our goal here will be more modest than those of \cite{recht2012beneath,gurbuzbalaban2015random}: Rather than show superiority of without-replacement sampling, we only show that it will not be significantly worse (in a worst-case sense) than with-replacement sampling. Nevertheless, such guarantees are novel, and still justify the use of with-replacement sampling, especially in situations where it is advantageous due to data access constraints or other reasons. Moreover, these results have a useful application in the context of distributed learning and optimization, as we will shortly describe.

Our main contributions can be summarized as follows:
\begin{itemize}
	\item For convex functions on some convex domain $\Wcal$, we consider algorithms which perform a single pass over a random permutation of $m$ individual functions, and show that their suboptimality can be characterized by a combination of two quantities, each from a different field: First, the \emph{regret} which the algorithm can attain in the setting of \emph{adversarial online convex optimization} \cite{cesa2006prediction,shalev2011online,hazan2015introduction}, and second, the \emph{transductive Rademacher complexity} of $\Wcal$ with respect to the individual functions, a notion stemming from transductive learning theory \cite{vapnik1998statistical,elyanivpech09}.
	\item As a concrete application of the above, we show that if each function $f_i(\cdot)$ corresponds to a convex Lipschitz loss of a linear predictor, and the algorithm belongs to the class of algorithms which in the online setting attain $\Ocal(\sqrt{T})$ regret on $T$ such functions (which includes, for example, stochastic gradient descent), then the suboptimality using without-replacement sampling, after processing $T$ functions, is $\Ocal(1/\sqrt{T})$. Up to constants, the guarantee is the same as that obtained using with-replacement sampling.
	\item We turn to consider more specifically the stochastic gradient descent algorithm, and show that if the objective function $F(\cdot)$ is $\lambda$-strongly convex, and the functions $f_i(\cdot)$ are also smooth, then the suboptimality bound becomes $\Ocal(1/\lambda T)$, which matches the with-replacement guarantees (although with replacement, smoothness is not needed, and the dependence on some parameters hidden in the $\Ocal(\cdot)$ is somewhat better). 
	\item In recent years, a new set of fast stochastic algorithms to solve \eqref{eq:obj} has emerged, such as SAG \cite{roux2012stochastic}, SDCA (as analyzed in \cite{shalev2013stochastic,shalev2016accelerated}, SVRG \cite{johnson2013accelerating}, SAGA \cite{defazio2014saga}, Finito  \cite{defazio2014finito}, S2GD \cite{konevcny2013semi} and others. These algorithms are characterized by cheap stochastic iterations, involving computations of individual function gradients, yet unlike traditional stochastic algorithms, enjoy a linear convergence rate (runtime scaling logarithmically with the required accuracy). To the best of our knowledge, all existing analyses require sampling with replacement. We consider a representative algorithm from this set, namely SVRG, and the problem of regularized least squares, and show that similar guarantees can be obtained using without-replacement sampling. This result has a potentially interesting implication: Under the mild assumption that the problem's condition number is smaller than the data size, we get that SVRG can converge to an arbitrarily accurate solution (even up to machine precision), \emph{without} the need to reshuffle the data -- only a single shuffle at the beginning suffices. Thus, at least for this problem, we can obatin fast and high-quality solutions even if random data access is expensive. 
	\item A further application of the SVRG result is in the context of distributed learning: By simulating without-replacement SVRG on data randomly partitioned between several machines, we get a nearly-optimal algorithm for regularized least squares, in terms of communication and computational complexity, as long as the condition number is smaller than the data size per machine (up to logarithmic factors). This builds on the work of Lee et al. \cite{lee2015distributed}, who were the first to recognize the applicability of SVRG to distributed  optimization. However, their results relied on with-replacement sampling, and are applicable only for much smaller condition numbers. 
\end{itemize}

We note that our focus is on scenarios where no reshufflings are necessary. In particular, the $\Ocal(1/\sqrt{T})$ and $\Ocal(1/\lambda T)$ bounds apply for all $T\in \{1,2,\ldots,m\}$, namely up to one full pass over a random permutation of the entire data. However, our techniques are also applicable to a constant $(>1)$ number of passes, by randomly reshuffling the data after every pass. In a similar vein, our SVRG result can be readily extended to a situation where each epoch of the algorithm is done on an independent permutation of the data. We leave a full treatment of this to future work. 

The paper is organized as follows. In \secref{sec:prelim}, we introduce the basic notation and definitions used throughout the paper. In \secref{sec:convex}, we discuss the case of convex and Lipschitz functions. In \secref{sec:strconvex}, we consider strongly convex functions, and in \secref{sec:svrg}, we provide the analysis for the SVRG algorithm. We provide full proofs in \secref{sec:proofs}, and conclude with a summary and open problems in \secref{sec:summary}. Some technical proofs and results are relegated to appendices.

\section{Preliminaries and Notation}\label{sec:prelim}

We use bold-face symbols to denote vectors. Given a vector $\bw$, $w_i$ denotes it's $i$-th coordinate. We utilize the standard $\Ocal(\cdot),\Theta(\cdot),\Omega(\cdot)$ notation to hide constants, and $\tilde{\Ocal},\tilde{\Theta}(\cdot),\tilde{\Omega}(\cdot)$ to hide constants and logarithmic factors.

Given convex functions $f_1(\cdot),f_2(\cdot),\ldots,f_m(\cdot)$ from $\reals^d$ to $\reals$, we define our objective function $F:\reals^d\rightarrow\reals$ as 
\[
F(\bw)=\frac{1}{m}\sum_{i=1}^{m}f_i(\bw),
\]
with some fixed optimum $\bw^*\in \arg\min_{\bw\in\Wcal}F(\bw)$. In machine learning applications, each individual $f_i(\cdot)$ usually corresponds to a loss with respect to a data point, hence will use the terms ``individual function'', ``loss function'' and ``data point'' interchangeably throughout the paper.

We let $\sigma$ be a permutation over $\{1,\ldots,m\}$ chosen uniformly at random. In much of the paper, we consider methods which draw loss functions without replacement according to that permutation (that is, $f_{\sigma(1)}(\cdot),f_{\sigma(2)}(\cdot),f_{\sigma(3)}(\cdot),\ldots$). We will use the shorthand notation
\[
F_{1:t-1}(\cdot) = \frac{1}{t-1}\sum_{i=1}^{t-1}f_{\sigma(i)}(\cdot)~~,~~
F_{t:m}(\cdot) = \frac{1}{m-t+1}\sum_{i=t}^{m}f_{\sigma(i)}(\cdot)
\]
to denote the average loss with respect to the first $t-1$ and last $m-t+1$ loss functions respectively, as ordered by the permutation (intuitively, the losses in $F_{1:t-1}(\cdot)$ are those already observed by the algorithm at the beginning of iteration $t$, whereas the losses in $F_{t:m}(\cdot)$ are those not yet observed). We use the convention that $F_{1:1}(\cdot) \equiv 0$, and the same goes for other expressions of the form $\frac{1}{t-1}\sum_{i=1}^{t-1}\cdots$ throughout the paper, when $t=1$. We also define quantities such as $\nabla F_{1:t-1}(\cdot)$ and $\nabla F_{t:m}(\cdot)$ similarly. 

A function $f:\reals^d\rightarrow \reals$ is \emph{$\lambda$-strongly convex}, if for any $\bw,\bw'$, 
\[
f(\bw')~\geq~ f(\bw)+\inner{\bg,\bw'-\bw}+\frac{\lambda}{2}\norm{\bw'-\bw}^2,
\]
where $\bg$ is any (sub)-gradient of $f$ at $\bw$. Note that any convex function is $0$-strongly convex. We also say a  function $f$ is \emph{$\mu$-smooth} if for any $\bw,\bw'$, 
\[
f(\bw')~\leq~ f(\bw)+\inner{\bg,\bw'-\bw}+\frac{\mu}{2}\norm{\bw'-\bw}^2.
\]
$\mu$-smoothness also implies that the function $f$ is differentiable, and its gradient is $\mu$-Lipschitz. Based on these properties, it is easy to verify that if $\bw^*\in \arg\min f(\bw)$, and $f$ is $\lambda$-strongly convex and $\mu$-smooth, then 
\[
\frac{\lambda}{2}\norm{\bw-\bw^*}^2~\leq~ f(\bw)-f(\bw^*)~\leq~ \frac{\mu}{2}\norm{\bw-\bw^*}^2.
\]

We will also require the notion of trandsuctive Rademacher complexity, as developed by El-Yaniv and Pechyony \cite[Definition 1]{elyanivpech09}, with a slightly different notation adapted to our setting:

\begin{definition}
	Let $\Vcal$ be a set of vectors $\bv=(v_1,\ldots,v_m)$ in $\reals^m$. Let $s,u$ be positive integers such that $s+u=m$, and denote $p:=\frac{su}{(s+u)^2}\in (0,1/2)$. We define
	the transductive Rademacher Complexity $\Rcal_{s,u}(\Vcal)$ as
	\[
	\Rcal_{s,u}(\Vcal) = \left(\frac{1}{s}+\frac{1}{u}\right)\cdot\E_{r_1,\ldots,r_m}\left[\sup_{\bv\in\Vcal}\sum_{i=1}^{m}r_i v_i  \right],
	\]
	where $r_1,\ldots,r_m$ are i.i.d. random variables such that 
	\[
	r_i = \begin{cases} 1 & \text{w.p.}~p \\ -1 & \text{w.p.}~p \\ 0 & \text{w.p.}~1-2p\end{cases}
	\]
\end{definition}
This quantity is an important parameter is studying the richness of the set $\Vcal$,  and will prove crucial in providing some of the convergence results presented later on. Note that it differs slightly from standard Rademacher complexity, which is used in the theory of learning from i.i.d. data (e.g. \cite{bartlett2003rademacher,mohri2012foundations,shalev2014understanding}, where the Rademacher variables $r_i$ only take $-1,+1$ values, and $(1/s+1/u)$ is replaced by $1/m$).

\section{Convex Lipschitz Functions}\label{sec:convex}

We begin by considering loss functions $f_1(\cdot),f_2(\cdot),\ldots,f_m(\cdot)$ which are convex and $L$-Lipschitz over some convex domain $\Wcal$. We assume the algorithm sequentially goes over some permuted ordering of the losses, and before processing the $t$-th loss function, produces an iterate $\bw_t\in \Wcal$. Moreover, we assume that the algorithm has a regret bound in the adversarial online setting, namely that for \emph{any} sequence of $T$ convex Lipschitz losses $f_1(\cdot),\ldots,f_T(\cdot)$, and any $\bw\in\Wcal$,
\[
\sum_{t=1}^{T}f_{t}(\bw_t)-\sum_{t=1}^{T}f_{t}(\bw)\leq R_T
\]
for some $R_T$ scaling sub-linearly in $T$\footnote{For simplicity, we assume the algorithm is deterministic given $f_1,\ldots,f_m$, but all results in this section also hold for randomized algorithms (in expectation over the algorithm's randomness), assuming the \emph{expected} regret of the algorithm w.r.t. any $\bw\in\Wcal$ is at most $R_T$.}. For example, online gradient descent (which is equivalent to stochastic gradient descent when the losses are i.i.d.), with a suitable step size, satisfies $R_T=\Ocal(BL\sqrt{T})$, where $L$ is the Lipschitz parameter and $B$ upper bounds the norm of any vector in $\Wcal$ \cite{zinkevich2003online}. A similar regret bound can also be shown for other online algorithms (see \cite{hazan2015introduction,shalev2011online,xiao2010dual}).

Since the ideas used for analyzing this setting will also be used in the more complicated results which follow, we provide the analysis in some detail. We first have the following straightforward theorem, which upper bounds the expected suboptimality in terms of regret and the expected difference between the average loss on prefixes and suffixes of the data.

\begin{theorem}\label{thm:regret}
	Suppose the algorithm has a regret bound $R_T$, and sequentially processes   $f_{\sigma(1)}(\cdot),\ldots, f_{\sigma(T)}(\cdot)$ where $\sigma$ is a random permutation on $\{1,\ldots,m\}$. Then in expectation over $\sigma$,
\[
\E\left[\frac{1}{T}\sum_{t=1}^{T}F(\bw_t)-F(\bw^*)\right] ~\leq~\frac{R_T}{T}+\frac{1}{mT}\sum_{t=2}^{T}(t-1)\cdot \E[F_{1:t-1}(\bw_t)-F_{t:m}(\bw_t)].
\]
\end{theorem}
The left hand side in the inequality above can be interpreted as an expected bound on $F(\bw_t)-F(\bw^*)$, where $t$ is drawn uniformly at random from $1,2,\ldots,T$. Alternatively, by Jensen's inequality and the fact that $F(\cdot)$ is convex, the same bound also applies on $\E[F(\bar{\bw}_T)-F(\bw^*)]$, where $\bar{\bw}_T=\frac{1}{T}\sum_{t=1}^{T}\bw_t$.

The proof of the theorem relies on the following simple but key lemma, which expresses the expected difference between with-replacement and without-replacement sampling in an alternative form, similar to \thmref{thm:regret} and one which lends itself to tools and ideas from transductive learning theory. This lemma will be used in proving all our main results.

\begin{lemma}\label{lem:key}
	Let $\sigma$ be a permutation over $\{1,\ldots,m\}$ chosen uniformly at random. Let $s_1,\ldots,s_m\in \reals$ be random variables which conditioned on $\sigma(1),\ldots,\sigma(t-1)$, are independent of $\sigma(t),\ldots,\sigma(m)$. Then
	\[
	\E\left[\frac{1}{m}\sum_{i=1}^{m}s_i-s_{\sigma(t)}\right] = \frac{t-1}{m}\cdot \E\left[s_{1:t-1}-s_{t:m}\right]
	\]
	for $t>1$, and for $t=1$ we have $\E\left[\frac{1}{m}\sum_{i=1}^{m}s_i-s_{\sigma(t)}\right]=0$.
\end{lemma}
The proof of the Lemma appears in \subsecref{subsec:prooflemkey}. We now turn to prove \thmref{thm:regret}:

\begin{proof}[Proof of \thmref{thm:regret}]
	Adding and subtracting terms, and using the facts that $\sigma$ is a permutation chosen uniformly at random, and $\bw^*$ is fixed, 
	\begin{align*}
	\E\left[\frac{1}{T}\sum_{t=1}^{T}F(\bw_t)-F(\bw^*)\right] &~=~
	\E\left[\frac{1}{T}\sum_{t=1}^{T}\left(f_{\sigma(t)}(\bw_t)-F(\bw^*)\right)\right]+\E\left[\frac{1}{T}\sum_{t=1}^{T}\left(F(\bw_t)-f_{\sigma(t)}(\bw_t)\right)\right]\\
	&~=~
	\E\left[\frac{1}{T}\sum_{t=1}^{T}\left(f_{\sigma(t)}(\bw_t)-f_{\sigma(t)}(\bw^*)\right)\right]+\E\left[\frac{1}{T}\sum_{t=1}^{T}\left(F(\bw_t)-f_{\sigma(t)}(\bw_t)\right)\right]\\ 
	\end{align*}
	Applying the regret bound assumption on the sequence of losses $f_{\sigma(1)}(\cdot),\ldots,f_{\sigma(T)}(\cdot)$, the above is at most
	\[
	\frac{R_T}{T}+\frac{1}{T}\sum_{t=1}^{T}\E\left[F(\bw_t)-f_{\sigma(t)}(\bw_t)\right].
	\]
	Since $\bw_t$ (as a random variable over the permutation $\sigma$ of the data) depends only on $\sigma(1),\ldots,\sigma(t-1)$, we can use \lemref{lem:key} (where $s_i=f_i(\bw_t)$, and noting that the expectation above is $0$ when $t=1$), and get that the above equals
	\[
	\frac{R_T}{T}+\frac{1}{mT}\sum_{t=2}^{T}(t-1)\cdot \E[F_{1:t-1}(\bw_t)-F_{t:m}(\bw_t)].
	\]
\end{proof}

Having reduced the expected suboptimality to the expected difference $\E[F_{1:t-1}(\bw_t)-F_{t:m}(\bw_t)]$, the next step is to upper bound it with 
$\E[\sup_{\bw\in \Wcal} \left(F_{1:t-1}(\bw)-F_{t:m}(\bw)\right)]$: Namely, having split our loss functions at random to two groups of size $t-1$ and $m-t+1$, how large can be the difference between the average loss of any $\bw$ on the two groups? Such uniform convergence bounds are exactly the type studied in transductive learning theory, where a fixed dataset is randomly split to a training set and a test set, and we consider the generalization performance of learning algorithms ran on the training set. Such results can be provided in terms of the transductive Rademacher complexity of $\Wcal$, and combined with \thmref{thm:regret}, lead to the following bound in our setting:
\begin{theorem}\label{thm:sqrtt}
		Suppose that each $\bw_t$ is chosen from a fixed domain $\Wcal$,  that the algorithm enjoys a regret bound $R_T$, and that $\sup_{i,\bw\in\Wcal}|f_i(\bw)|\leq B$. Then in expectation over the random permutation $\sigma$,
		\[
		\E\left[\frac{1}{T}\sum_{t=1}^{T}F(\bw_t)-F(\bw^*)\right]~\leq~
		\frac{R_T}{T}+\frac{1}{mT}\sum_{t=2}^{T}(t-1)\Rcal_{t-1:m-t+1}(\Vcal)+\frac{24B}{\sqrt{m}},		
		\]
		where $\Vcal=\left\{(f_1(\bw),\ldots,f_m(\bw)~|~\bw\in\Wcal\right\}$.
\end{theorem}
Thus, we obtained a generic bound which depends on the online learning characteristics of the algorithm, as well as the statistical  learning properties of the class $\Wcal$ on the loss functions. The proof (as the proofs of all our results from now on) appears in \secref{sec:proofs}.

We now instantiate the theorem to the prototypical case of bounded-norm linear predictors, where the loss is some convex and Lipschitz function of the prediction $\inner{\bw,\bx}$ of a predictor $\bw$ on an instance $\bx$, possibly with some regularization:

\begin{corollary}\label{cor:sqrtt}
	Under the conditions of \thmref{thm:sqrtt}, suppose that $\Wcal\subseteq \{\bw:\norm{\bw}\leq \bar{B}\}$, and each loss function $f_i$ has the form $\ell_i(\inner{\bw,\bx_i})+r(\bw)$ for some $L$-Lipschitz $\ell_i$, $\norm{\bx_i}\leq 1$, and a fixed function $r$. Then
	\[
	\E\left[\frac{1}{T}\sum_{t=1}^{T}F(\bw_t)-F(\bw^*)\right]~\leq~\frac{R_T}{T}+\frac{2\left(12+\sqrt{2}\right)\bar{B}L}{\sqrt{m}}.
	\]
\end{corollary}

As discussed earlier, in the setting of Corollary \ref{cor:sqrtt}, typical regret bounds are on the order of $\Ocal(\bar{B}L\sqrt{T})$. Thus, the expected suboptimality is $\Ocal(\bar{B}L/\sqrt{T})$, all the way up to $T=m$ (i.e. a full pass over a random permutation of the data). Up to constants, this is the same as the suboptimality attained by $T$ iterations of with-replacement sampling, using stochastic gradient descent or similar algorithms \cite{shalev2014understanding}.

\section{Faster Convergence for Strongly Convex Functions}
\label{sec:strconvex}

We now consider more specifically the stochastic gradient descent algorithm on a convex domain $\Wcal$, which can be described as follows: We initialize at some $\bw_1\in \Wcal$, and perform the update steps
\[
\bw_{t+1}=\Pi_{\Wcal}(\bw_t-\eta_t \bg_t),
\]
where $\eta_t>0$ are fixed step sizes, $\Pi_{\Wcal}$ is projection on $\Wcal$, and $\bg_t$ is a subgradient of $f_{\sigma(t)}(\cdot)$ at $\bw_t$. Moreover, we assume the objective function $F(\cdot)$ is $\lambda$-strongly convex for some $\lambda>0$. In this scenario, using with-replacement sampling (i.e. $\bg_t$ is a subgradient of an independently drawn $f_i(\cdot)$), performing $T$ iterations as above and returning a randomly sampled iterate $\bw_t$ or their average results in expected suboptimality $\tilde{\Ocal}(G^2/\lambda T)$, where $G^2$ is an upper bound on the expected squared norm of $\bg_t$ \cite{rakhlin2011making,shalev2014understanding}. Here, we study a similar situation in the without-replacement case. 

In the result below, we will consider specifically the case where each $f_i(\bw)$ is a Lipschitz and smooth loss of a linear predictor $\bw$, possibly with some regularization. The smoothness assumption is needed to get a good bound on the transductive Rademacher complexity of quantities such as $\inner{\nabla f_i(\bw),\bw}$. However, the technique can be potentially applied to more general cases. 

\begin{theorem}\label{thm:1t}
	Suppose $\Wcal$ has diameter $B$, and that $F(\cdot)$ is $\lambda$-strongly convex on $\Wcal$. Assume that  $f_i(\bw)=\ell_i(\inner{\bw,\bx_i})+r(\bw)$ where $\norm{\bx_i}\leq 1$, $r(\cdot)$ is possibly some regularization term, and each $\ell_i$ is $L$-Lipschitz and $\mu$-smooth on $\{z:z=\inner{\bw,\bx},\bw\in\Wcal,\norm{\bx}\leq 1\}$. Furthermore, suppose $\sup_{\bw\in\Wcal}\norm{\nabla f_i(\bw)}\leq G$. Then for any $1<T\leq m$, if we run SGD for $T$ iterations with step size $\eta_t=2/\lambda t$, we have
\begin{align*}
	\E\left[\frac{1}{T}\sum_{t=1}^{T}F(\bw_t)-F(\bw^*)\right]&~\leq~
	c\cdot\frac{\left(L+\mu B\right)^2\left(1+\log\left(\frac{m}{m-T+1}\right)\right)+G^2\log(T)}{\lambda T}\\
	&\leq~ c\cdot \frac{((L+\mu B)^2 + G^2)\log(T)}{\lambda T},
\end{align*}
where $c$ is a universal positive constant.
\end{theorem}
As in the results of the previous section, the left hand side is the expected optimality of a single $\bw_t$ where $t$ is chosen uniformly at random, or an upper bound on the expected suboptimality of the average $\bar{\bw}_T=\frac{1}{T}\sum_{t=1}^{T}\bw_t$. This result is similar to the with-replacement case, up to numerical constants and the additional $(L+\mu B^2)$ term in the numerator. We note that in some cases, $G^2$ is the dominant term anyway\footnote{$G$ is generally on the order of $L+\lambda B$, which is the same as $L+\mu B$ if $L$ is the dominant term. This happens for instance with the squared loss, whose Lipschitz parameter is on the order of $\mu B$.}. However, it is not clear that the $(L+\mu B^2)$ term is necessary, and removing it is left to future work.

\begin{remark}[$\log(T)$ Factor]
	It is possible to remove the logarithmic factors in \thmref{thm:1t}, by considering not $\frac{1}{T}\sum_{t=1}^{T}F(\bw_t)$, but rather only an average over a suffix of the iterates ($\bw_{\epsilon T},\bw_{\epsilon T+1}\ldots,\bw_{T}$ for some fixed $\epsilon\in (0,1)$), or by a weighted averaging scheme. This has been shown in the case of with-replacement sampling (\cite{rakhlin2011making,lacoste2012simpler,shamir2013stochastic}), and the same proof techniques can be applied here. 
\end{remark}

The proof of \thmref{thm:1t} is somewhat more challenging than the results of the previous section, since we are attempting to get a faster rate of $\Ocal(1/T)$ rather than $\Ocal(1/\sqrt{T})$, all the way up to $T=m$. A significant technical obstacle is that our proof technique relies on concentration of averages around expectations, which on $T$ samples does not go down faster than $\Ocal(1/\sqrt{T})$. To overcome this, we apply concentration results not on the function values (i.e. $F_{1:t-1}(\bw_t)-F_{t:m}(\bw_t)$ as in the previous section), but rather on gradient-iterate inner products, i.e. $\inner{\nabla F_{1:t-1}(\bw_t)-\nabla F_{t:m}(\bw_t),\bw_t-\bw^*}$, where $\bw^*$ is the optimal solution. To get good bounds, we need to assume these gradients have a certain structure, which is why we need to make the assumption in the theorem about the form of each $f_i(\cdot)$. Using transductive Rademacher complexity tools, we manage to upper bound the expectation of these inner products by quantities roughly of the form $\sqrt{\E[\norm{\bw_t-\bw^*}^2]}/\sqrt{t}$ (assuming here $t<m/2$ for simplicity). We now utilize the fact that in the strongly convex case, $\norm{\bw_t-\bw^*}$ itself decreases to zero with $t$ at a certain rate, to get fast rates decreasing as $1/t$. The full proof appears in \subsecref{subsec:proofthm1t}.

\section{Without-Replacement SVRG for Least Squares}\label{sec:svrg}

In this section, we will consider a more sophisticated stochastic gradient approach, namely the SVRG algorithm of \cite{johnson2013accelerating}, designed to solve optimization problems with a finite sum structure as in \eqref{eq:obj}. Unlike purely stochastic gradient procedures, this algorithm does require several passes over the data. However, assuming the condition number $\mu/\lambda$ is smaller than the data size (where $\mu$ is the smoothness of each $f_i(\cdot)$, and $\lambda$ is the strong convexity parameter of $F(\cdot)$), only $\Ocal(m\log(1/\epsilon))$ gradient evaluations are required to get an $\epsilon$-accurate solution, for any $\epsilon$. Thus, we can get a high-accuracy solution after the equivalent of a small number of data passes. As discussed in the introduction, over the past few years several other algorithms have been introduced and shown to have such a behavior. We will focus on the algorithm in its basic form, where the domain $\Wcal$ is unconstrained and equals $\reals^d$.

The existing analysis of SVRG (\cite{johnson2013accelerating}) assumes stochastic iterations, which involves sampling the data with replacement. Thus, it is natural to consider whether a similar convergence behavior occurs using without-replacement sampling. As we shall see, a positive reply has at least two implications: The first is that as long as the condition number is smaller than the data size, SVRG can be used to obtain a high-accuracy solution, \emph{without} the need to reshuffle the data: Only a single shuffle at the beginning suffices, and the algorithm terminates after a small number of \emph{sequential} passes (logarithmic in the required accuracy). The second implication is that such without-replacement SVRG can be used to get a nearly-optimal algorithm for convex distributed learning and optimization on randomly partitioned data, as long as the condition number is smaller than the data size at each machine. 

The SVRG algorithm using without-replacement sampling on a dataset of size $m$ is described as Algorithm \ref{alg:svrg}. It is composed of multiple epochs (indexed by $s$), each involving one gradient computation on the entire dataset, and $T$ stochastic iterations, involving gradient computations with respect to individual data points. Although the  gradient computation on the entire dataset is expensive, it is only needed to be done once per epoch. Since the algorithm will be shown to have linear convergence as a function of the number of epochs, this requires only a small (logarithmic) number of passes over the data.

\begin{algorithm}\caption{SVRG using Without-Replacement Sampling}
\begin{algorithmic}
	\STATE \textbf{Parameters:} $\eta,T$, permutation $\sigma$ on $\{1,\ldots,m\}$
	\STATE Initialize $\tilde{\bw}_1$ at $\mathbf{0}$
	\FOR{$s=1,2,\ldots$}
		\STATE $\bw_{(s-1)T+1}~:=~\tilde{\bw}_s$
		\STATE $\tilde{\bn} ~:=~ \nabla F(\tilde{\bw}_s) ~=~ \frac{1}{m}\sum_{i=1}^{m}\nabla f_i(\tilde{\bw}_s)$
		\FOR{$t=(s-1)T+1,\ldots,sT$}
			\STATE $\bw_{t+1} ~:=~ \bw_t-\eta \left(\nabla f_{\sigma(t)}(\bw_t)-\nabla f_{\sigma(t)}(\tilde{\bw}_s)+\tilde{\bn}\right)$
		\ENDFOR
		\STATE Let $\tilde{\bw}_{s+1}$ be the average of $\bw_{(s-1)T+1},\ldots,\bw_{sT}$, or one of them drawn uniformly at random.
	\ENDFOR
\end{algorithmic}\label{alg:svrg}
\end{algorithm}

We will consider specifically the regularized least mean squares problem, where
\begin{equation}\label{eq:squared}
f_i(\bw) = \frac{1}{2}\left(\inner{\bw,\bx_i}-y_i\right)^2+\frac{\hat{\lambda}}{2}\norm{\bw}^2.
\end{equation}
for some $\bx_i,y_i$ and $\hat{\lambda}>0$. Moreover, we assume that $F(\bw)=\frac{1}{m}\sum_{i=1}^{m}f_i(\bw)$ is $\lambda$-strongly convex (note that necessarily $\lambda\geq \hat{\lambda}$). For convenience, we will assume that $\norm{\bx_i},|y_i|,\lambda$ are all at most $1$ (this is without much loss of generality, since we can always re-scale the loss functions by an appropriate factor). Note that under this assumption, each $f_i(\cdot)$ as well as $F(\cdot)$ are also $1+\hat{\lambda} \leq 2$-smooth.

\begin{theorem}\label{thm:svrg}	
	Suppose each loss function $f_i(\cdot)$ corresponds to \eqref{eq:squared}, where $\bx_i\in \reals^d$, $\max_i \norm{\bx_i}\leq 1$, $\max_i |y_i|\leq 1$, $\hat{\lambda}>0$, and that $F(\cdot)$ is $\lambda$-strongly convex, where $\lambda\in (0,1)$. Moreover, let $B\geq 1$ be such that $\norm{\bw^*}^2\leq B$ and $\max_t F(\bw_t)-F(\bw^*)\leq B$ with probability $1$ over the random permutation.
	 
	There is some universal constant $c_0\geq 1$, such that for any $c\geq c_0$ and any $\epsilon\in (0,1)$, if we run algorithm \ref{alg:svrg}, using parameters $\eta,T$ satisfying
	\[
	\eta = \frac{1}{c}~~,~~
	T\geq \frac{9}{\eta\lambda}~~,~~m\geq c\log^2\left(\frac{64dmB^2}{\lambda\epsilon}\right)T,
	\]
	then after $S=\lceil \log_4(9/\epsilon)\rceil$ epochs of the algorithm, we obtain a point $\tilde{\bw}_{S+1}$ for which
	\[
	\E[F(\tilde{\bw}_{S+1})-\min_{\bw}F(\bw)] ~\leq~ \epsilon.
	\]
\end{theorem}

In particular, by taking $\eta=\Theta(1)$ and $T=\Theta(1/\lambda)$, the algorithm attains an $\epsilon$-accurate solution after $\Ocal\left(\frac{\log(1/\epsilon)}{\lambda}\right)$ stochastic iterations of without-replacement sampling, and $\Ocal(\log(1/\epsilon))$ sequential passes over the data to compute gradients of $F(\cdot)$. This implies that as long as $1/\lambda$ (which stands for the condition number of the problem) is smaller than $\Ocal(m/\log(1/\epsilon))$, the number of without-replacement stochastic iterations is smaller than the data size $m$. Therefore, assuming the data is randomly shuffled, we can get a solution using only \emph{sequential} passes over the data, without the need to reshuffle.

We make a few remarks about the result:

\begin{remark}[$B$ Parameter]
The condition that $\max_t F(\bw_t)-F(\bw^*)\leq B$ with probability $1$ is needed for technical reasons in our analysis, which requires some uniform upper bound on $F(\bw_t)-F(\bw^*)$. That being said, $B$ only appears inside log factors, so even a very crude bound would suffice. In appendix \ref{app:B}, we indeed show that under the theorem's conditions, there is always a valid $B$ satisfying $\log(B)= \Ocal\left(\log(1/\epsilon)\log(T)+\log(1/\lambda)\right)$, which can be plugged into \thmref{thm:svrg}. Even then, we conjecture that this bound on $\max_t F(\bw_t)-F(\bw^*)$ can be significantly improved, or perhaps even removed altogether from the analysis. 
\end{remark}

\begin{remark}[$\log(d)$ Factor]
	The logarithmic dependence on the dimension $d$ is due to an application of a matrix Bernstein inequality for without-replacement sampling on $d\times d$ matrices. However, with some additional effort, it might be possible to shave off this log factor and get a fully dimension-free result, by deriving an appropriate matrix concentration result depending on the intrinsic dimension (see Section 7 in \cite{tropp2015introduction} for derivations in the with-replacement setting). We did not systematically pursue this, as it is not the focus of our work. 
\end{remark}

\begin{remark}[Extension to other losses]
	Our proof applies as-is in the slightly more general setting where each $f_i(\cdot)$ is of the form $f_i(\bw)=\frac{1}{2}\bw^\top A_i \bw+\bb_i^\top \bw+c_i$, where $\norm{A_i},\norm{\bb_i}$ are bounded and $A_i$ is a rank-1 positive semidefinite matrix (i.e. of the form $\bu_i\bu_i^\top$ for some vector $\bu_i$). The rank-1 assumption is required in one step of the proof, when deriving \eqref{eq:mi2}. Using a somewhat different analysis, we can also derive a convergence bound for general strongly convex and smooth losses, but only under the condition that the epoch size $T$ scales as $\Omega(1/\lambda^2)$ (rather than  $\Omega(1/\lambda)$), which means that the strong convexity parameter $\lambda$ can be at most $\tilde{\Ocal}(1/\sqrt{m})$ (ignoring logarithmic factors). Unfortunately, in the context of without-replacement sampling, this is a somewhat trivial regime: For such $\lambda$, the analysis of with-replacement SVRG requires sampling less than $\tilde{\Ocal}(\sqrt{m})$ individual loss functions for convergence. But by the birthday paradox, with-replacement and without-replacement sampling are then statistically indistinguishable, so we could just argue directly that the expected suboptimality of without-replacement SVRG would be similar to that of with-replacement SVRG.
\end{remark}


The proof of \thmref{thm:svrg} (in \subsecref{subsec:proofthmsvrg}) uses ideas similar to those employed in the previous section, but is yet again more technically challenging. Recall that in the previous section, we were able to upper bound certain inner products by a quantity involving both the iteration number $t$ and $\norm{\bw_t-\bw^*}^2$, namely the Euclidean distance of the iterate $\bw_t$ to the optimal solution. To get \thmref{thm:svrg}, this is not good enough, and we need to upper bound similar inner products directly in terms of the suboptimality $F(\bw_t)-F(\bw^*)$, as well as $F(\tilde{\bw}_s)-F(\bw^*)$. We are currently able to get a satisfactory result for least squares, where the Hessians of the objective function are fixed, and the concentration tools required boil down to concentration of certain stochastic matrices, without the need to apply transductive Rademacher complexity results. The full proof appears in \subsecref{subsec:proofthmsvrg}.

\subsection{Application of Without-Replacement SVRG to distributed learning}

An important variant of the problems we discussed so far is when training data (or equivalently, the individual loss functions $f_1(\cdot),\ldots,f_m(\cdot)$) are split between different machines, who need to communicate in order to reach a good solution. This naturally models situations where data is too large to fit at a single machine, or where we wish to speed up the optimization process by parallelizing it on several computers. There has been much research devoted to this question in the past few years, with some examples including 
\cite{zhang2013communication,bekkerman2011scaling,balcandistributed,ZWSL11,AgChDuLa11,BoPaChPeEc11,MaKeSyBo13,Ya13,DeGiShXi12,CoShSrSr11,DuAgWa12,jaggi2014communication,shamir2014communication,shamir2014distributed,zhang2015communication,lee2015distributed}. 

A substantial number of these works focus on the setting where the data is split equally \emph{at random} between $k$ machines (or equivalently, that data is assigned to each machine by sampling without replacement from $\{f_1(\cdot),\ldots,f_m(\cdot)\}$)). Intuitively, this creates statistical similarities between the data at different machines, which can be leveraged to aid the optimization process. Recently, Lee et al. \cite{lee2015distributed} proposed a  simple and elegant solution, at least in certain parameter regimes. This is based on the observation that SVRG, according to its with-replacement analysis, requires $\Ocal(\log(1/\epsilon))$ epochs, where in each epoch one needs to compute an exact gradient of the objective function $F(\cdot)=\frac{1}{m}\sum_{i=1}^{m}f_i(\cdot)$, and $\Ocal(\mu/\lambda)$ gradients of individual losses $f_i(\cdot)$ chosen uniformly at random (where $\epsilon$ is the required suboptimality, $\mu$ is the smoothness parameter of each $f_i(\cdot)$, and $\lambda$ is the strong convexity parameter of $F(\cdot)$). Therefore, if each machine had \emph{i.i.d.} samples from $\{f_1(\cdot),\ldots,f_m(\cdot)\}$, whose union cover $\{f_1(\cdot),\ldots,f_m(\cdot)\}$, the machines could just simulate SVRG: First, each machine splits its data to batches of size $\Ocal(\mu/\lambda)$. Then, each SVRG epoch is simulated by the machines computing a gradient of $F(\cdot)$, involving a fully parallel computation and one communication round\footnote{Specifically,  each machine computes a gradient with respect to an average of a different subset of $\{f_1(\cdot),\ldots,f_m(\cdot)\}$, which can be done in parallel, and then perform distributed averaging to get a gradient with respect to $F(\cdot)=\frac{1}{m}\sum_{i=1}^{m}f_i(\cdot)$. For simplicity we assume a broadcast model where this requires a single communication round, but this can be easily generalized.}, and one machine computing gradients with respect to one of its batches. Overall, this would require $\log(1/\epsilon)$ communication rounds, and $\Ocal(m/k+\mu/\lambda)\log(1/\epsilon)$ runtime, where $m/k$ is the number of data points per machine (ignoring communication time, and assuming constant time to compute a gradient of $f_i(\cdot)$). Under the reasonable assumption that the condition number $\mu/\lambda$ is less than $m/k$, this requires runtime linear in the data size per machine, and logarithmic in the target accuracy $\epsilon$, with only a logarithmic number of communication rounds. Up to logarithmic factors, this is essentially the best one can hope for with a distributed algorithm\footnote{Moreover, a lower bound in \cite{lee2015distributed} implies that under very mild conditions, the number of required communication rounds is at least $\tilde{\Omega}(\sqrt{(\mu/\lambda)/(m/k)})$. Thus, a logarithmic number of communication rounds is unlikely to be possible once $\mu/\lambda$ is significantly larger than $m/k$.}.
	
Unfortunately, the reasoning above crucially relies on each machine having access to i.i.d. samples, which can be reasonable in some cases, but is different than the more common assumption that the data is randomly \emph{split} among the machines. To circumvent this issue, \cite{lee2015distributed} propose to communicate individual data points / losses $f_i(\cdot)$ between machines, so as to simulate i.i.d. sampling. However, by the birthday paradox, this only works well in the regime where the overall number of samples required ($\Ocal((\mu/\lambda)\log(1/\epsilon)$) is not much larger than $\sqrt{m}$. Otherwise, with-replacement and without-replacement sampling becomes statistically very different, and a large number of data points would need to be communicated. In other words, if communication is an expensive resource, then the solution proposed in \cite{lee2015distributed} only works well up to condition number $\mu/\lambda$ being roughly $\sqrt{m}$. In machine learning applications, the strong convexity parameter $\lambda$ often comes from explicit regularization designed to prevent over-fitting, and needs to scale with the data size, usually so that $1/\lambda$ is between $\sqrt{m}$ and $m$. Thus, the solution above is communication-efficient only when $1/\lambda$ is relatively small

However, the situation immediately improves if we can use a \emph{without-replacement} version of SVRG, which can easily be simulated with randomly partitioned data: The stochastic batches can now be simply subsets of each machine's data, which are statistically identical to sampling $\{f_1(\cdot),\ldots,f_m(\cdot)\}$ without replacement. Thus, no data points need to be sent across machines, even if $1/\lambda$ is large. We present the pseudocode as Algorithm \ref{alg:distsvrg}.

\begin{algorithm}\caption{Distributed Without-Replacement SVRG}
	\begin{algorithmic}
		\STATE \textbf{Parameters:} $\eta,T$
		\STATE \textbf{Assume:} $\{f_1(\cdot),\ldots,f_m(\cdot)\}$ randomly split to machines $1,2,\ldots,n$ (possibly different number at different machines)
		\STATE Each machine $j$ splits its data arbitrarily to $b_j$ batches $B^j_1,\ldots,B^j_{b_i}$ of size $T$ 
		\STATE $j:=1~,~k:=1~,~t:=1$
		\STATE All machines initialize $\tilde{\bw}_1$ at $\mathbf{0}$
		\FOR{$s=1,2,\ldots,$}
		\STATE Perform communication round to compute $\tilde{\bn} ~:=~ \nabla F(\tilde{\bw}_s) ~=~ \frac{1}{m}\sum_{i=1}^{m}\nabla f_i(\tilde{\bw}_s)$
		\STATE Machine $j$ performs $\bw_{1}~:=~\tilde{\bw}_s$		
		\FOR{Each $f$ in $B^j_k$}
		\STATE Machine $j$ performs $\bw_{t+1} ~:=~ \bw_t-\eta \left(\nabla f\bw_t)-\nabla f(\tilde{\bw}_s)+\tilde{\bn}\right)$
		\STATE $t:=t+1$
		\ENDFOR
		\STATE Machine $j$ computes $\tilde{\bw}_{s+1}$ as average of $\bw_{1},\ldots,\bw_{T}$, or one of them drawn uniformly at random.
		\STATE Perform communication round to distribute $\tilde{\bw}_{s+1}$ to all machines
		\STATE $k:=k+1$
		\STATE If $k>b_j$, let $k:=1,j:=j+1$
		\ENDFOR
	\end{algorithmic}\label{alg:distsvrg}
\end{algorithm}

Let us consider the analysis of no-replacement SVRG provided in \thmref{thm:svrg}. According to this analysis, by setting $T=\Theta(1/\lambda)$, then as long as the total number of batches is at least $\Omega(\log(1/\epsilon))$, and $1/\lambda = \tilde{\Ocal}(m)$, then Algorithm \ref{alg:distsvrg} will attain an $\epsilon$-suboptimal solution in expectation. In other words, without any additional communication, we extend the applicability of distributed SVRG (at least for regularized least squares) from the $1/\lambda=\tilde{\Ocal}(\sqrt{m})$ regime to $1/\lambda=\tilde{\Ocal}(m)$, which covers most scenarios relevant to machine learning. 

We emphasize that this formal analysis only applies to regularized squared loss, which is the scope of \thmref{thm:svrg}. However, Algorithm \ref{alg:distsvrg} can be applied to any loss function, and we conjecture that it will have similar performance for any smooth losses and strongly-convex objectives.

\section{Proofs}\label{sec:proofs}

Before providing the proofs of our main theorems, we develop in Subsection \ref{subsec:rademacher} below some important technical results on transductive Rademacher complexity, required in some of the proofs.

\subsection{Results on Transductive Rademacher Complexity}\label{subsec:rademacher}

In this subsection, we develop a few important tools and result about Rademacher complexity, that will come handy in our analysis. We begin by the following theorem, which is a straightforward corollary of Theorem 1 from \cite{elyanivpech09} (attained by simplifying and upper-bounding some of the parameters).

\begin{theorem}\label{thm:rad}
	Suppose $\Vcal \subseteq [-B,B]^m$ for some $B>0$. Let $\sigma$ be a permutation over $\{1,\ldots,m\}$ chosen uniformly at random, and define $\bv_{1:s}=\frac{1}{s}\sum_{j=1}^{s}v_{\sigma(j)}$, $\bv_{s+1:m}=\frac{1}{u}\sum_{j=s+1}^{m}v_{\sigma(j)}$. Then for any $\delta\in (0,1)$, with probability at least $1-\delta$,
	\[
	\sup_{\bv\in\Vcal} \left(\bv_{1:s} - \bv_{s+1:m}\right)~\leq~ \Rcal_{s,u}(\Vcal)+6B\left(\frac{1}{\sqrt{s}}+\frac{1}{\sqrt{u}}\right)\left(1+\log\left(\frac{1}{\delta}\right)\right).
	\]
\end{theorem}
We note that the theorem in \cite{elyanivpech09} actually bounds $ \bv_{s+1:m}-\bv_{1:s}$, but the proof is completely symmetric to the roles of $\bv_{s+1:m}$ and $\bv_{1:s}$, and hence also  implies the formulation above.

We will also need the well-known contraction property, which states that the Rademacher complexity of a class of vectors $\Vcal$ can only increase by a factor of $L$ if we apply on each coordinate a fixed $L$-Lipschitz function:
\begin{lemma}\label{lem:contraction}
	Let $g_1,\ldots,g_m$ are real-valued, $L$-Lipschitz functions, and given some $\Vcal\subseteq \reals^m$, define $\bg\circ\Vcal=\{(g_1(v_1),\ldots,g_m(v_m)):(v_1,\ldots,v_m)\in \Vcal\}$. Then
	\[
	\Rcal_{s,u}(\bg\circ\Vcal)\leq L\cdot\Rcal_{s,u}(\Vcal).
	\]
\end{lemma}
This is a slight generalization of Lemma 5 from \cite{elyanivpech09} (which is stated for $g_1=g_2=\ldots=g_m$, but the proof is exactly the same). 

In our analysis, we will actually only need bounds on the expectation of $\bv_{1:s} - \bv_{s+1:m}$, which is weaker than what \thmref{thm:rad} provides. Although such a bound can be developed from scratch, we find it more convenient to simply get such a bound from \thmref{thm:rad}. Specifically, combining \thmref{thm:rad} and \lemref{lem:expprob} from Appendix \ref{app:technical}, we have the following straightforward corollary:

\begin{corollary}\label{cor:rad}
	Suppose $\Vcal \subseteq [-B,B]^m$ for some $B>0$. Let $\sigma$ be a permutation over $\{1,\ldots,m\}$ chosen uniformly at random, and define $\bv_{1:t-1}=\frac{1}{t-1}\sum_{j=1}^{t-1}v_{\sigma(j)}$, $\bv_{t:m}=\frac{1}{m-t+1}\sum_{j=t}^{m}v_{\sigma(j)}$.
	Then 
	\[
	\E\left[\sup_{\bv\in\Vcal}
	\bv_{1:t-1}-\bv_{t:m}\right]
	~\leq~ \Rcal_{t-1,m-t+1}(\Vcal)+12B\left(\frac{1}{\sqrt{t-1}}+\frac{1}{\sqrt{m-t+1}}\right).
	\]
	Moreover, if $\sup_{\bv\in\Vcal} \left(\bv_{1:t-1}-\bv_{t:m}\right)\geq 0$ for any permutation $\sigma$, then
	\[
	\sqrt{\E\left[\left(\sup_{\bv\in\Vcal} \bv_{1:t-1}-\bv_{t:m}\right)^2\right]} ~\leq~ \sqrt{2}\cdot\Rcal_{s,u}(\Vcal)+12\sqrt{2}B\left(\frac{1}{\sqrt{t-1}}+\frac{1}{\sqrt{m-t+1}}\right).
	\]
\end{corollary}

We now turn to collect a few other structural results, which will be useful when studying the Rademacher complexity of linear predictors or loss gradients derived from such predictors.

\begin{lemma}\label{lem:radtwo}
	Given two sets of vectors $\Vcal\in [-B_{\Vcal},B_{\Vcal}]^m$,$\Scal \subseteq [-B_{\Scal},B_{\Scal}]^m$ for some $B_\Vcal,B_\Scal\geq 0$, define 
	\[
	\Ucal = \left\{\left(v_1s_1,\ldots,v_m s_m\right)~:~(v_1,\ldots,v_m)\in\Vcal,(s_1,\ldots,s_m)\in \Scal\right\}.
	\]
	Then
	\[
	\Rcal_{s,u}(\Ucal) \leq B_{\Scal}\cdot \Rcal_{s,u}(\Vcal)+B_\Vcal\cdot\Rcal_{s,u}(\Scal).
	\]
\end{lemma}
\begin{proof}
	The proof resembles the proof of the contraction inequality for standard Rademacher complexity (see for instance Lemma 26.9 in \cite{shalev2014understanding}).
	
	By definition of $\Rcal_{s,u}$, it is enough to prove that
	\begin{equation}\label{eq:radtwo}
	\E_{r_1,\ldots,r_m}\left[\sup_{\bv,\bs}\sum_{i=1}^{m}r_iv_is_i \right]~\leq~\E_{r_1,\ldots,r_m}\left[\sup_{\bv,\bs}\sum_{i=1}^{m}r_i(B_\Scal\cdot v_i+B_\Vcal\cdot s_i)\right],
	\end{equation}
	since the right hand side can be upper bounded by $B_\Scal\cdot \E[\sup_{\bv}\sum_{i=1}^{m}r_iv_i]+B_\Vcal\cdot\E[\sup_{\bs}\sum_{i=1}^{m}r_i s_i]$. To get this, we will treat the coordinates one-by-one, starting with the first coordinate and showing that
	\begin{equation}\label{eq:prodz}
	\E_{r_1,\ldots,r_m}\left[\sup_{\bv,\bs}\sum_{i=1}^{m}r_i v_i s_i  \right]
	~\leq~
	\E_{r_1,\ldots,r_m}\left[\sup_{\bv,\bs}\left( r_1  (B_\Scal\cdot v_1+B_\Vcal\cdot s_1)+\sum_{i=2}^{m}r_iv_is_i\right)\right].
	\end{equation}
	Repeating the same argument for coordinates $2,3,\ldots,m$ will yield \eqref{eq:radtwo}.
	
	For any values $v,v'$ and $s,s'$ in the coordinates of some $\bv\in\Vcal$ and $\bs\in\Scal$ respectively, we have
	\begin{align}
	|vs-v's'| &~=~ |vs-v's+v's-v's'|~ \leq~ |vs-v's|+|v's-v's'|\notag\\
	&~\leq~
	|v-v'|\cdot|s|+|s-s'|\cdot|v| ~\leq~ B_\Scal |v-v'|+B_\Vcal |s-s'|.
	\label{eq:philip}
	\end{align}
	Recalling that $r_i$ are i.i.d. and take values of $+1$ and $-1$ with probability $p$ (and $0$ otherwise), we can write the left hand side of \eqref{eq:prodz} as
	\begin{align*}
	&\E_{r_1,\ldots,r_m}\left[\sup_{\bv,\bs}\left(r_1 v_1s_1+\sum_{i=2}^{m}r_i v_is_i\right)  \right]\\
	&= 	\E_{r_2,\ldots,r_m}\left[p\cdot\sup_{\bv,\bs}\left( v_1s_1+\sum_{i=2}^{m}r_iv_is_i\right)+p\cdot\sup_{\bv,\bs}\left( -v_1s_1+\sum_{i=2}^{m}r_iv_is_i\right)  \right.\\
	&~~~~~~~~~~~~~~~~~~~~~~~~\left.+(1-2p)\sup_{\bv,\bs}\left(\sum_{i=2}^{m}r_iv_is_i\right)\right]\\
	&= 	\E_{r_2,\ldots,r_m}\left[\sup_{\bv,\bs}\left( p~v_1s_1+p \sum_{i=2}^{m}r_iv_is_i\right)+\sup_{\bv',\bs'}\left( -p~v_1s_1+p~ \sum_{i=2}^{m}r_iv'_is'_i\right)  \right.\\
	&~~~~~~~~~~~~~~~~~~~~~~~~\left.+\sup_{\bv'',\bs''}\left((1-2p)
	\sum_{i=2}^{m}r_i v''_i s''_i\right)\right]\\	
	&= 	\E_{r_2,\ldots,r_m}\left[\sup_{\bv,\bv',\bv'',\bs,\bs'\bs''}\left( p(v_1s_1-v'_1s'_1)+p\sum_{i=2}^{m}r_iv_is_i
	+p\sum_{i=2}^{m}r_iv'_is'_i\right.\right.\\
	&~~~~~~~~~~~~~~~~~~~~~~~~\left.\left.+(1-2p)\sum_{i=2}^{m}r_i v''_i s''_i)\right)\right].
	\end{align*}
	Using \eqref{eq:philip} and the fact that we take a supremum over $\bv,\bv'$ and $\bs,\bs'$ from the same sets, this equals
	\begin{align*}
	&= 	\E_{r_2,\ldots,r_m}\left[\sup_{\bv,\bv',\bv'',\bs,\bs'\bs''}\left( p\left(B_\Scal|v_1-v'_1|+B_\Vcal|s_1-s'_1|\right)+p\sum_{i=2}^{m}r_iv_is_i
	+p\sum_{i=2}^{m}r_iv'_is'_i\right.\right.\\
	&~~~~~~~~~~~~~~~~~~~~~~~~\left.\left.+(1-2p)\sum_{i=2}^{m}r_iv''_i s''_i\right)\right]\\
	&= 	\E_{r_2,\ldots,r_m}\left[\sup_{\bv,\bv',\bv'',\bs,\bs'\bs''}\left( p\left(B_\Scal(v_1-v'_1)+B_\Vcal(s_1-s'_1)\right)+p\sum_{i=2}^{m}r_iv_is_i
	+p\sum_{i=2}^{m}r_iv'_is'_i\right.\right.\\
	&~~~~~~~~~~~~~~~~~~~~~~~~\left.\left.+(1-2p)\sum_{i=2}^{m}r_iv''_i s''_i\right)\right]\\
	&= 	\E_{r_2,\ldots,r_m}\left[p~\sup_{\bv,\bs}\left( (B_\Scal\cdot v_1+B_\Vcal\cdot s_1)+\sum_{i=2}^{m}r_iv_is_i\right)+p~\sup_{\bv',\bs'}\left(
	p(-B_\Scal v'_1-B_\Vcal s'_1)
	+p\sum_{i=2}^{m}r_iv'_is'_i\right)\right.\\
	&~~~~~~~~~~~~~~~~~~~~~~~~\left.+(1-2p)\sup_{\bv,\bs''}\left(\sum_{i=2}^{m}r_iv''_i s''_i\right)\right]\\
	&=\E_{r_1,\ldots,r_m}\left[\sup_{\bv,\bs}\left( r_1 (B_\Scal\cdot v_1+B_\Vcal\cdot s_1)+\sum_{i=2}^{m}r_iv_is_i\right)\right]
	\end{align*}
	as required.
\end{proof}

\begin{lemma}\label{lem:linear}
	For some $B>0$ and vectors $\bx_1,\ldots,\bx_m$ with Euclidean norm at most $1$, let 
	\[
	\Vcal_B = \{(\inner{\bw,\bx_1},\ldots,\inner{\bw,\bx_m}):\norm{\bw}\leq B\}.
	\]
	Then
	\[
	\Rcal_{s,u}(\Vcal_B) \leq \sqrt{2}B\left(\frac{1}{\sqrt{s}}+\frac{1}{\sqrt{u}}\right).
	\]
\end{lemma}
\begin{proof}
	Using the definition of $\Vcal_B$ and applying Cauchy-Schwartz,
	\begin{align*}
	\E_{r_1,\ldots,r_m}&\left[\sup_{\bv\in\Vcal}\sum_{i=1}^{m}r_i v_i\right]
	~=~
	\E_{r_1,\ldots,r_m}\left[\sup_{\bw:\norm{\bw}\leq B}\left\langle \bw,\sum_{i=1}^{m}r_i \bx_i\right\rangle\right]~\leq~
	\E_{r_1,\ldots,r_m}\left[B\left\|\sum_{i=1}^{m}r_i \bx_i\right\|\right]\\	
	&\leq B\sqrt{\E_{r_1,\ldots,r_m}\left[\left\|\sum_{i=1}^{m}r_i \bx_i\right\|^2\right]}
	~=~ B\sqrt{\E_{r_1,\ldots,r_m}\sum_{i,j=1}^{m}r_i r_j \inner{\bx_i,\bx_j}}.
	\end{align*}
	Recall that $r_i$ are independent and equal $+1,-1$ with probability $p$ (and $0$ otherwise). Therefore, for $i\neq j$, $\E[r_i r_j]=0$, and if $i=j$, $\E[r_i r_j]=\E[r_i^2] = 2p$. Using this and the assumption that $\norm{\bx_i}\leq 1$ for all $i$, the above equals
	\[
	B\sqrt{\sum_{i=1}^{m}2p \inner{\bx_i,\bx_i}}~\leq~ B\sqrt{2pm}.
	\]
	Therefore, $\Rcal_{s,u}(\Vcal_B) \leq \left(\frac{1}{s}+\frac{1}{u}\right)B\sqrt{2pm}$. 
	Recalling that $p = \frac{su}{(s+u)^2}$ where $s+u=m$ and plugging it in, we get the upper bound
	\begin{align*}
	&B\left(\frac{1}{s}+\frac{1}{u}\right)\sqrt{2\frac{su}{(s+u)^2}(s+u)}
	~=~ \sqrt{2}B\left(\frac{1}{s}+\frac{1}{u}\right)\sqrt{\frac{su}{s+u}}\\
	&=~ \sqrt{2}B\left(\frac{1}{s}\sqrt{s\frac{u}{s+u}}+\frac{1}{u}\sqrt{u\frac{s}{s+u}}\right)
	~\leq~ \sqrt{2}B\left(\frac{1}{\sqrt{s}}+\frac{1}{\sqrt{u}}\right),
	\end{align*}
	from which the result follows.
\end{proof}

Combining \corref{cor:rad}, \lemref{lem:contraction}, \lemref{lem:radtwo} and \lemref{lem:linear}, we have the following:

\begin{corollary}\label{cor:linear}
	Suppose the functions $f_1(\cdot),\ldots,f_m(\cdot)$ are of the form $f_i(\bw)=\ell_i(\inner{\bw,\bx_i})+r(\bw)$, where  $\norm{\bx_i}\leq 1$ and $r$ is some fixed function. Let $B>0$ such that the iterates $\bw$ chosen by the algorithm are from a set $\Wcal\subseteq \{\bw:\norm{\bw}\leq B\}$ which contains the origin $\mathbf{0}$. Finally, assume $\ell_i(\cdot)$ is $L$-Lipschitz and $\mu$-smooth over the interval $[-B,B]$. Then
	\begin{align*}
	&\sqrt{\E\left[\left(\sup_{\bw\in\Wcal} \left\langle\nabla F_{1:t-1}(\bw)-\nabla F_{t:m}(\bw)~,~\frac{\bw}{\norm{\bw}}\right\rangle\right)^2\right]}\\
	&~~~~~~~~~~ ~\leq~ \left(19L+2\mu B\right)\left(\frac{1}{\sqrt{t-1}}+\frac{1}{\sqrt{m-t+1}}\right).
	\end{align*}
\end{corollary}
\begin{proof}
	By definition of $f_i$, we have
	\[
	\inner{\nabla f_i(\bw),\frac{\bw}{\norm{\bw}}} = \ell'_i(\inner{\bw,\bx_i})\inner{\frac{\bw}{\norm{\bw}},\bx_i}+\inner{\nabla r(\bw),\frac{\bw}{\norm{\bw}}}.
	\]
	Therefore, the expression in the corollary statement can be written as
	\[
	\sqrt{\E\left[\left(\sup_{\bu\in\Ucal} u_{1:t-1}-u_{t:m}\right)^2\right]},
	\]
	where 
	\[
	u_i = \ell'_i(\inner{\bw,\bx_i})\inner{\frac{\bw}{\norm{\bw}},\bx_i}
	\]
	and
	\[
	\Ucal = \left\{\left(\ell'_1(\inner{\bw,\bx_1})\inner{\frac{\bw}{\norm{\bw}},\bx_1},\ldots,\ell'_m(\inner{\bw,\bx_m})\inner{\frac{\bw}{\norm{\bw}},\bx_m}\right)~:~\norm{\bw}\in\Wcal \right\}
	\]
	(note that the terms involving $r$ get cancelled out in the expression $u_{1:s}-u_{s+1:m}$, so we may drop them). Applying \corref{cor:rad} (noting that $|u_i|\leq L$, and that $\sup_{\bu\in\Ucal} u_{1:t-1}-u_{t:m}\geq 0$, since $u_{1:t-1}-u_{t:m}=0$ if we choose $\bw=\mathbf{0}$), we have
	\begin{equation}\label{eq:corlinear0}
	\sqrt{\E\left[\left(\sup_{\bu\in\Ucal} u_{1:t-1}-u_{t:m}\right)^2\right]} ~\leq~  \sqrt{2}\cdot\Rcal_{t-1,m-t+1}(\Ucal)+12\sqrt{2}L\left(\frac{1}{\sqrt{t-1}}+\frac{1}{\sqrt{m-t+1}}\right).
	\end{equation}
	Now, define
	\[
	\Vcal =\left\{\left(\ell'_1(\inner{\bw,\bx_1}),\ldots,\ell'_m(\inner{\bw,\bx_m})\right)~:~\norm{\bw}\leq B\right\}
	\]
	and
	\[
	\Scal = \left\{\left(\inner{\bw,\bx_1},\ldots,\inner{\bw,\bx_m}\right)~:~\norm{\bw}= 1\right\},
	\]
	and note that $\Ucal$ as we defined it satisfies
	\[
	\Ucal ~\subseteq~ \left\{(\ell'_1(v_1)s_1,\ldots,\ell'_m(v_m)s_m):(v_1,\ldots,v_m)\in\Vcal,(s_1,\ldots,s_m)\in \Scal\right\}.
	\]
	Moreover, by construction, the coordinates of each $\bv\in\Vcal$ are bounded in $[-L,L]$, and the coordinates of each $\bs\in\Scal$ are bounded in $[-1,+1]$. Applying \lemref{lem:radtwo}, we get
	\begin{equation}\label{eq:corlinear1}
	\Rcal_{t-1,m-t+1}(\Ucal) ~\leq~ \Rcal_{t-1,m-t+1}(\Vcal)+L\cdot \Rcal_{t-1,m-t+1}(\Scal).
	\end{equation}
	Using \lemref{lem:linear}, we have
	\begin{equation}\label{eq:corlinear2}
	\Rcal_{t-1,m-t+1}(\Scal) ~\leq~ \sqrt{2}\left(\frac{1}{\sqrt{t-1}}+\frac{1}{\sqrt{m-t+1}}\right).
	\end{equation}	
	Finally, applying \lemref{lem:contraction} (using the fact that each $\ell'_i$ is $\mu$-Lipschitz) followed by \lemref{lem:linear}, we have
	\begin{equation}\label{eq:corlinear3}
	\Rcal_{t-1,m-t+1}(\Vcal) ~\leq~ \sqrt{2}\mu B\left(\frac{1}{\sqrt{t-1}}+\frac{1}{\sqrt{m-t+1}}\right).
	\end{equation}	
	Combining \eqref{eq:corlinear1}, \eqref{eq:corlinear2} and \eqref{eq:corlinear3}, plugging into \eqref{eq:corlinear0}, and slightly simplifying for readability, yields the desired result.
\end{proof}

\subsection{Proof of \lemref{lem:key}}\label{subsec:prooflemkey}
	The lemma is immediate when $t=1$, so we will assume $t>1$. Also, we will prove it when the expectation $\E$ is conditioned on $\sigma(1),\ldots,\sigma(t-1)$, and the result will follow by taking expectations over them. With this conditioning, $s_1,\ldots,s_m$ have some fixed distribution, which is independent of how $\sigma$ permutes $\{1,\ldots,m\}\setminus \{\sigma(1),\ldots,\sigma(t-1)\}$.
	
	Recall that $\sigma$ is chosen uniformly at random. Therefore, conditioned on $\sigma(1),\ldots,\sigma(t-1)$, the value of $\sigma(t)$ is uniformly distributed on $\{1,\ldots,m\}\setminus \{\sigma(1),\ldots,\sigma(t-1)\}$, which is the same set as $\sigma(t),\ldots,\sigma(m)$. Therefore, the left hand side in the lemma statement equals
	\begin{align*}
	&\E\left[\frac{1}{m}\sum_{i=1}^{m}s_i-\frac{1}{m-t+1}\sum_{i=t}^{m}s_{\sigma(i)}\right]\\
	&=\E\left[\frac{1}{m}\sum_{i=1}^{m}s_{\sigma(i)}-\frac{1}{m-t+1}\sum_{i=t}^{m}s_{\sigma(i)}\right]\\
	&=\E\left[\frac{1}{m}\sum_{i=1}^{t-1}s_{\sigma(i)}+\left(\frac{1}{m}-\frac{1}{m-t+1}\right)\sum_{i=t}^{m}s_{\sigma(i)}\right]\\
	&=\E\left[\frac{1}{m}\sum_{i=1}^{t-1}s_{\sigma(i)}-\frac{t-1}{m(m-t+1)}\sum_{i=t}^{m}s_{\sigma(i)}\right]\\
	&=
	\frac{t-1}{m}\cdot \E\left[\frac{1}{t-1}\sum_{i=1}^{t-1}s_{\sigma(i)}-\frac{1}{m-t+1}\sum_{i=t}^{m}s_{\sigma(i)}\right]
	\end{align*}
	as required.

\subsection{Proof of \thmref{thm:sqrtt}}
	Let $\Vcal=\{(f_1(\bw),\ldots,f_m(\bw))~|~\bw\in\Wcal\}$ and applying \corref{cor:rad}, we have
	\begin{align*}
	\E[F_{1:t-1}(\bw_t)-F_{t:m}(\bw_t)] &~\leq~ \E\left[\sup_{\bw\in\Wcal}F_{1:t-1}(\bw)-F_{t:m}(\bw)\right]\\ &~\leq~\Rcal_{t-1:m-t+1}(\Vcal)+12B\left(\frac{1}{\sqrt{t-1}}+\frac{1}{\sqrt{m-t+1}}\right).
	\end{align*}
	Plugging this into the bound from \thmref{thm:regret}, we have
	\[
	\E\left[F(\bar{\bw}_T)-F(\bw^*)\right]~\leq~\frac{R_T}{T}+\frac{1}{mT}\sum_{t=2}^{T}(t-1)\left(\Rcal_{t-1:m-t+1}(\Vcal)+12B\left(\frac{1}{\sqrt{t-1}}+\frac{1}{\sqrt{m-t+1}}\right)\right)
	\]
	Applying \lemref{lem:tsqrtbound}, the right hand side is at most
	\begin{equation}\label{eq:firstpart}
	\frac{R_T}{T}+\frac{1}{mT}\sum_{t=2}^{T}(t-1)\Rcal_{t-1:m-t+1}(\Vcal)+\frac{24B}{\sqrt{m}}.
	\end{equation}

\subsection{Proof of Corollary \ref{cor:sqrtt}}

Note that all terms in the bound of \thmref{thm:sqrtt}, except the regret term, are obtained by considering the difference $F_{1:t-1}(\bw_t)-F_{t:m}(\bw_t)$, so any additive terms in the losses which are constant (independent of $i$) are cancelled out. Therefore, we may assume without loss of generality that $f_i(\bw)=\ell_i(\inner{\bw,\bx_i})$ (without the $r(\bw)$ term), and that $\ell_i$ is centered so that $\ell_i(0)=0$. Applying \lemref{lem:contraction} and \lemref{lem:linear}, we can upper the Rademacher complexity as follows:
\begin{align*}
\Rcal_{t-1:m-t+1}(\Vcal) &~\leq~ L\cdot \Rcal_{t-1:m-t+1}\left(\left\{\left(\inner{\bw,\bx_1},\ldots,\inner{\bw,\bx_m}\right)~|~\bw\in\Wcal\right\}\right)\\
&~\leq~ \sqrt{2}\cdot BL\left(\frac{1}{\sqrt{t-1}}+\frac{1}{\sqrt{m-t+1}}\right).
\end{align*}
Plugging this into \thmref{thm:sqrtt}, applying \lemref{lem:tsqrtbound}, and noting that by the assumptions above and in the corollary statement, $\sup_{i,\bw\in\Wcal}|f_i(\bw)|\leq\sup_{a\in[-\bar{B},\bar{B}]}|\ell_i(a)|\leq \bar{B}L$, the result follows.

\subsection{Proof of \thmref{thm:1t}}\label{subsec:proofthm1t}
	Since the algorithm is invariant to shifting the coordinates or shifting all loss functions by a constant, we will assume without loss of generality that $\Wcal$ contains the origin $\mathbf{0}$ (and therefore $\Wcal\subseteq \{\bw:\norm{\bw}\leq B\}$), that the objective function $F(\cdot)$ is minimized at $\mathbf{0}$, and that $F(\mathbf{0})=0$. By definition of the algorithm and convexity of $\Wcal$, we have
	\begin{align*}
	\E[\norm{\bw_{t+1}}^2] &= \E\left[\norm{\Pi_{\Wcal}(\bw_t-\eta_t \nabla f_{\sigma(t)}(\bw_t))}^2\right]~\leq~ \E\left[\norm{\bw_t-\eta_t \nabla f_{\sigma(t)}(\bw_t)}^2\right]\\
	&\leq \E\left[\norm{\bw_t}^2\right]-2\eta_t\E\left[\inner{\nabla f_{\sigma(t)}(\bw_t),\bw_t}\right]+\eta_t^2G^2\\
	&= 
	\E\left[\norm{\bw_t}^2\right]-2\eta_t\E\left[\inner{\nabla F(\bw_t),\bw_t}\right]+2\eta_t\E\left[\inner{\nabla F(\bw_t)-\nabla f_{\sigma(t)}(\bw_t),\bw_t}\right]+\eta_t^2G^2.	
	\end{align*}
	By definition of strong convexity, since $F(\cdot)$ is $\lambda$-strongly convex, minimized at $\mathbf{0}$, and assumed to equal $0$ there, we have $\inner{\nabla F(\bw_t),\bw_t}\geq F(\bw_t)+\frac{\lambda}{2}\norm{\bw}^2$. Plugging this in, changing sides and dividing by $2\eta_t$, we get
	\begin{equation}\label{eq:1t0}
	\E[F(\bw_t)] ~\leq~ \left(\frac{1}{2\eta_t}- \frac{\lambda}{2}\right)\E[\norm{\bw_t}^2]-\frac{1}{2\eta_t}\cdot\E[\norm{\bw_{t+1}}^2] +\E\left[\inner{\nabla F(\bw_t)-\nabla f_{\sigma(t)}(\bw_t),\bw_t}\right]+\frac{\eta_t}{2} G^2.
	\end{equation}
	We now turn to treat the third term in the right hand side above. 
	Since $\bw_t$ (as a random variable over the permutation $\sigma$ of the data) depends only on $\sigma(1),\ldots,\sigma(t-1)$, we can use \lemref{lem:key} and Cauchy-Schwartz, to get
	\begin{align*}
	& \E\left[\inner{\nabla F(\bw_t)-\nabla f_{\sigma(t)}(\bw_t),\bw_t}\right]~=~\E\left[\inner{\frac{1}{m}\sum_{i=1}^{m}\nabla f_i(\bw_t)-\nabla f_{\sigma(t)}(\bw_t),\bw_t}\right]\\
	&~=~
	\frac{t-1}{m}\cdot\E\left[\left(\inner{\nabla F_{1:t-1}(\bw_t)-\nabla F_{t:m}(\bw_t),\bw_t}\right)\right]\\
	&~=~\frac{t-1}{m}\cdot\E\left[\norm{\bw_t}\cdot\inner{\nabla F_{1:t-1}(\bw_t)-\nabla F_{t:m}(\bw_t),\frac{\bw_t}{\norm{\bw_t}}}\right]\\	
	&~\leq~
	\frac{t-1}{m}\cdot\E\left[\norm{\bw_t}\cdot \sup_{\bw\in\Wcal}\inner{\nabla F_{1:t-1}(\bw)-\nabla F_{t:m}(\bw),\frac{\bw}{\norm{\bw}}}\right]\\	
	&~\leq~
	\frac{t-1}{m}\cdot\sqrt{\E\left[\norm{\bw_t}^2\right]}\cdot\sqrt{\E\left[\left(\sup_{\bw\in\Wcal}\inner{\nabla F_{1:t-1}(\bw)-\nabla F_{t:m}(\bw),\frac{\bw}{\norm{\bw}}}\right)^2\right]}	
	\end{align*}
	Applying \corref{cor:linear} (using the convention $0/\sqrt{0}=0$ in the case $t=1$ where the expression above is $0$ anyway), this is at most
	\begin{align*}
	&\frac{t-1}{m}\cdot\sqrt{\E\left[\norm{\bw_t}^2\right]}\cdot\left(19L+2\mu B\right)\left(\frac{1}{\sqrt{t-1}}+\frac{1}{\sqrt{m-t+1}}\right)\\
	&~=~\sqrt{\E\left[\norm{\bw_t}^2\right]}\cdot\frac{19L+2\mu B}{m}\left(\sqrt{t-1}+\frac{t-1}{\sqrt{m-t+1}}\right).
	\end{align*}
	Using the fact that for any $a,b\geq 0$, $\sqrt{ab}=\sqrt{\frac{\lambda }{2}a\cdot \frac{2}{\lambda}b}\leq \frac{\lambda}{4} a + \frac{1}{\lambda}b$ by the arithmetic-geometric mean inequality, the above is at most
	\[
	\frac{\lambda}{4}\cdot\E[\norm{\bw_t}^2]+\frac{\left(19L+2\mu b\right)^2}{\lambda m^2}\left(\sqrt{t-1}+\frac{t-1}{\sqrt{m-t+1}}\right)^2.
	\]
	Since $(a+b)^2\leq 2(a^2+b^2)$, this is at most
	\[
	\frac{\lambda}{4}\cdot\E[\norm{\bw_t}^2]+\frac{2\left(19L+2\mu B\right)^2}{\lambda m^2}\left(t-1+\frac{(t-1)^2}{m-t+1}\right).
	\]	
	Plugging this back into \eqref{eq:1t0}, we get
	\[
	\E[F(\bw_t)] ~\leq~ \left(\frac{1}{2\eta_t}- \frac{\lambda}{4}\right)\E[\norm{\bw_t}^2]-\frac{1}{2\eta_t}\cdot\E[\norm{\bw_{t+1}}^2] +\frac{2\left(19L+2\mu B\right)^2}{\lambda m^2}\left(t-1+\frac{(t-1)^2}{m-t+1}\right)+\frac{\eta_t}{2} G^2.
	\]
	Averaging both sides over $t=1,\ldots,T$, and using Jensen's inequality, we have
	\begin{align*}
	&\E\left[\frac{1}{T}\sum_{t=1}^{T}F(\bw_t)\right]\\
	&~\leq~
	\frac{1}{2T}\sum_{t=1}^{T}\E[\norm{\bw_t}^2]\left(\frac{1}{\eta_t}-\frac{1}{\eta_{t-1}}-\frac{\lambda}{2}\right)+\frac{2\left(19L+2\mu B\right)^2}{\lambda m^2 T}\sum_{t=1}^{T}\left(t-1+\frac{(t-1)^2}{m-t+1}\right)+\frac{G^2}{2T}\sum_{t=1}^{T}\eta_t,
	\end{align*}
	where we use the convention that $1/\eta_0=0$. Since $T\leq m$, the second sum in the expression above equals
	\begin{align*}
	\sum_{t=0}^{T-1}\left(t+\frac{t^2}{m-t}\right) &~=~
	\sum_{t=0}^{T-1}t+\sum_{t=0}^{T-1}\frac{t^2}{m-t}
	~\leq~ \frac{T(T-1)}{2}+m^2\sum_{t=0}^{T-1}\frac{1}{m-t}\\
	&~\leq~ \frac{m^2}{2}+m^2\left(\sum_{t=0}^{T-2}\frac{1}{m-t}+1\right)
	~\leq~ \frac{3m^2}{2}+m^2\int_{t=0}^{T-1}\frac{1}{m-t}dt\\
	&~=~ m^2\left(\frac{3}{2}+\log\left(\frac{m}{m-T+1}\right)\right).
	\end{align*}
	Plugging this back in, we get
	\begin{align*}
	\E\left[\frac{1}{T}\sum_{t=1}^{T}F(\bw_t)\right]&~\leq~
	\frac{1}{2T}\sum_{t=1}^{T}\E[\norm{\bw_t}^2]\left(\frac{1}{\eta_t}-\frac{1}{\eta_{t-1}}-\frac{\lambda}{2}\right)\\
	&~~~~~~~+\frac{2\left(19L+2\mu B\right)^2\left(\frac{3}{2}+\log\left(\frac{m}{m-T+1}\right)\right)}{\lambda  T}
	+\frac{G^2}{T}\sum_{t=1}^{T}\eta_t.
	\end{align*}
	Now, choosing $\eta_t = 2/\lambda t$, and using the fact that $\sum_{t=1}^{T}\frac{1}{t}\leq \log(T)+1$, we get that
	\[
	\E[F(\bar{\bw}_T)]~\leq~
	\frac{2\left(19L+2\mu B\right)^2\left(\frac{3}{2}+\log\left(\frac{m}{m-T+1}\right)\right)}{\lambda  T}
	+\frac{2 G^2(\log(T)+1)}{\lambda T}.
	\]
	The result follows by a slight simplification, and recalling that we assumed $F(\bw^*)=F(\mathbf{0})=0$. The last inequality in the theorem is by the simple observation that $T(m-T+1)\geq m$ for any $T\in \{1,2,\ldots,m\}$, and therefore $\log(m/(m-T+1))\leq \log(T)$.

\subsection{Proof of \thmref{thm:svrg}}\label{subsec:proofthmsvrg}

The proof is based on propositions \ref{prop:vvbound_squared1} and \ref{prop:vvbound_squared2} presented below, which analyze the expectation of the update as well as its expected squared norm. The key technical challenge, required to get linear convergence, is to upper bound these quantities directly in terms of the suboptimality of the iterates $\bw_t,\tilde{\bw}_s$. To get Proposition \ref{prop:vvbound_squared1}, we state and prove a key lemma (\lemref{lem:matserf} below), which bounds the without-replacement concentration behavior of certain normalized stochastic matrices. The proof of \thmref{thm:svrg} itself is then a relatively straightforward calculation, relying on these results.

\begin{lemma}\label{lem:matserf}
	Let $\bx_1,\ldots,\bx_m$ be vectors in $\reals^d$ of norm at most $1$. Define 
	$\bar{X}=\frac{1}{m}\sum_{i=1}^{m}\bx_i\bx_i^\top$, and
	\[
	M_i = (\bar{X}+\hat{\gamma} I)^{-1/2}\bx_i\bx_i^\top(\bar{X}+\hat{\gamma} I)^{-1/2}
	\]
	for some $\gamma\geq 0$, so that $\bar{X}+\hat{\gamma} I$ has minimal eigenvalue $\gamma\in (0,1)$. Finally, let $\sigma$ be a permutation on $\{1,\ldots,m\}$ drawn uniformly at random. Then for any $\alpha \geq 2$, the probability
	\begin{align*}
	&\Pr\left(\exists s\in\{1,\ldots,m\}~:~\norm{\frac{1}{s}\sum_{i=1}^{s}M_{\sigma(i)}-\frac{1}{m-s}\sum_{i=s+1}^{m}M_{\sigma(i)}}\right.\\
	&\left.~~~~~~~~~~~~~~~~~~~~~~~~~~~~~~~~~~~~~~~~~~~~~~~~~~~>~ \frac{\alpha}{\sqrt{\gamma }}\left(\frac{1}{\sqrt{s}}+\frac{1}{\sqrt{m-s}}\right)+\frac{\alpha}{\gamma}\left(\frac{1}{s}+\frac{1}{m-s}\right)\right)
	\end{align*}
	is at most $4dm\exp\left(-\alpha/2\right)$.
\end{lemma}
\begin{proof}
	The proof relies on a without-replacement version of Bernstein's inequality for matrices (Theorem 1 in \cite{gross2010note}), which implies that for $d\times d$ Hermitican matrices $\hat{M}_i$ which satisfy
	\[
	\frac{1}{m}\sum_{i=1}^{m}\hat{M}_i=0 ~~,~~ \max_i \norm{\hat{M}_i}\leq c~~,~~ \norm{\frac{1}{m}\sum_{i=1}^{m}\hat{M}_i^2}\leq v,
	\]
	for some $v,c>0$, it holds that
	\begin{equation}\label{eq:serf}
	\Pr\left(\norm{\frac{1}{s}\sum_{i=1}^{s}\hat{M}_{\sigma(i)}}>z\right) ~\leq~
	\begin{cases} 2d\exp\left(-\frac{sz^2}{4v}\right) & z\leq 2v/c \\
	2d\exp\left(-\frac{sz}{2c}\right)& z> 2v/c \end{cases}
	\end{equation}
	In particular, we will apply this on the matrices
	\[
	\hat{M}_i ~=~ M_i -\frac{1}{m}\sum_{j=1}^{m}M_j ~=~ (\bar{X}+\gamma I)^{-1/2}\left(\bx_i\bx_i^\top-\bar{X}\right)(\bar{X}+\gamma I)^{-1/2}.
	\]
	Clearly, $\frac{1}{m}\sum_{i=1}^{m}\hat{M}_i=0$. We only need to find appropriate values for $v,c$.
	
	First, by definition of $\hat{M}_i$, we have
	\[
	\norm{\hat{M_i}} \leq \norm{(\bar{X}+\gamma I)^{-1/2}}\norm{\bx_i\bx_i-\bar{X}}\norm{(\bar{X}+\gamma I)^{-1/2}}.
	\]
	Since both $\bar{X}$ and $\bx_i\bx_i^\top$ are positive semidefinite and have spectral norm at most $1$, the above is at most $\gamma^{-1/2}\cdot 1 \cdot \gamma^{-1/2} = \gamma^{-1}$. Therefore, we can take $c=1/\gamma$.
	
	We now turn to compute an appropriate value for $v$. For convenience, let $\E$ denote a uniform distribution over the index $i=1,\ldots,m$, and note that $\E[\hat{M}_i]=0$. Therefore, we have
	\begin{align}
	\norm{\frac{1}{m}\sum_{i=1}^{m}\hat{M}_i^2} &= \norm{\E[\hat{M_i}^2]}~=~ \norm{\E[(M_i-\E[M_i])^2]}\notag\\
	&= \norm{\E[M_i^2]-\E^2[M_i]} ~\leq~ \max\{\norm{\E[M_i^2]},\norm{\E^2[M_i]}\},\label{eq:matserf1}
	\end{align}
	where in the last step we used the fact that $M_i$ is positive semidefinite.
	Let us first upper bound the second term in the max, namely
	\[
	\norm{\E^2[M_i]}~=~ \norm{\E[M_i]\cdot \E[M_i]}~\leq~\norm{\E[M_i]}^2~=~\norm{(\bar{X}+\gamma I)^{-1/2}\bar{X} (\bar{X}+\gamma I)^{-1/2}}^2.
	\]
	Since the expression above is invariant to rotating the positive semidefinite matrix $\bar{X}$, we can assume without loss of generality that $\bar{X}=\text{diag}(s_1,\ldots,s_d)$, in which case the above reduces to $
	\left(\max_i \frac{s_i}{s_i+\gamma}\right)^2 \leq 1$. Turning to the first term in the max in \eqref{eq:matserf1}, we have
	\begin{align}
	\norm{\E[M_i^2]} &=~ \norm{\frac{1}{m}\sum_{i=1}^{m}(\bar{X}+\gamma I)^{-1/2}\bx_i\bx_i^\top(\bar{X}+\gamma I)^{-1}\bx_i\bx_i^\top (\bar{X}+\gamma I)^{-1/2}}\notag\\
	&=~\norm{\frac{1}{m}\sum_{i=1}^{m}\left(\bx_i^\top(\bar{X}+\gamma I)^{-1}\bx_i\right)(\bar{X}+\gamma I)^{-1/2}\bx_i\bx_i^\top (\bar{X}+\gamma I)^{-1/2}}\notag\\
	&\stackrel{(1)}{\leq}~\norm{\frac{1}{m}\sum_{i=1}^{m}\norm{(\bar{X}+\gamma I)^{-1}}(\bar{X}+\gamma I)^{-1/2}\bx_i\bx_i^\top (\bar{X}+\gamma I)^{-1/2}}\notag\\
	&=~\norm{(\bar{X}+\gamma I)^{-1}}\norm{(\bar{X}+\gamma I)^{-1/2}\bar{X}(\bar{X}+\gamma I)^{-1/2}}\label{eq:mi2}	
	\end{align}
	where in $(1)$ we used the facts that $\norm{\bx_i}\leq 1$ and each term $(\bar{X}+\gamma I)^{-1/2}\bx_i\bx_i^\top (\bar{X}+\gamma I)^{-1/2}$ is positive semidefinite. As before, the expression above is invariant to rotating the positive semidefinite matrix $\bar{X}$, so we can assume without loss of generality that $\bar{X}=\text{diag}(s_1,\ldots,s_d)$, in which case the above reduces to
	\[
	\left(\max_i \frac{1}{s_i+\gamma}\right)\left(\max_i \frac{s_i}{s_i+\gamma}\right)~\leq~
	\frac{1}{\gamma}\cdot 1 ~=~ \frac{1}{\gamma}.
	\]
	Plugging these observations back into \eqref{eq:matserf1}, we get that
	\[
	\norm{\frac{1}{m}\sum_{i=1}^{m}\hat{M}_i^2} ~\leq~ \max\left\{1,\frac{1}{\gamma}\right\} = \frac{1}{\gamma}.
	\]
	Therefore, \eqref{eq:serf} applies with $v=c=1/\gamma$,, so we get that
	\[
	\Pr\left(\norm{\frac{1}{s}\sum_{i=1}^{s}\hat{M}_{\sigma(i)}}>z\right) ~\leq~
	\begin{cases} 2d\exp\left(-\frac{\gamma s z^2}{4}\right) & z\leq 2 \\
	2d\exp\left(-\frac{\gamma s z}{2}\right)& z>  2 \end{cases}.
	\]
	Substituting $z=\alpha\left(\frac{1}{\sqrt{\gamma s}}+\frac{1}{\gamma s}\right)$, we get that $\Pr\left(\norm{\frac{1}{s}\sum_{i=1}^{s}\hat{M}_{\sigma(i)}}>\alpha\left(\frac{1}{\sqrt{\gamma s}}+\frac{1}{\gamma s}\right)\right)$ can be upper bounded by
	\[
	2d\exp\left(-\frac{1}{4}\gamma s\alpha^2\left(\frac{1}{\sqrt{\gamma s}}+\frac{1}{\gamma s}\right) ^2\right)~\leq~ 2d\exp\left(-\frac{\alpha^2}{4}\right)
	\]		
	in the first case, and
	\[
	2d\exp\left(-\frac{1}{2}\gamma s\alpha\left(\frac{1}{\sqrt{\gamma s}}+\frac{1}{\gamma s}\right)\right)~\leq~ 2d\exp\left(-\frac{\alpha}{2}\right)	
	\]
	in the second case. Assuming $\alpha\geq 2$, both expressions can be upper bounded by $2d\exp\left(-\alpha/2\right)$, so we get that
	\[
	\Pr\left(\norm{\frac{1}{s}\sum_{i=1}^{s}\hat{M}_{\sigma(i)}}>\alpha\left(\frac{1}{\sqrt{\gamma s}}+\frac{1}{\gamma s}\right)\right) ~\leq~ 2d\exp\left(-\frac{\alpha}{2}\right)
	\]
	for any $\alpha\geq 2$. Recalling the definition of $\hat{M}_i$, we get
	\begin{equation}\label{eq:matref2}
	\Pr\left(\norm{\frac{1}{s}\sum_{i=1}^{s}M_{\sigma(i)}-\frac{1}{m}\sum_{i=1}^{m}M_{i}}>\alpha\left(\frac{1}{\sqrt{\gamma s}}+\frac{1}{\gamma s}\right)\right) ~\leq~ 2d\exp\left(-\frac{\alpha}{2}\right).
	\end{equation}
	Since the permutation is random, the exact same line of argument also works if we consider the last $m-s$ matrices rather than the first $s$ matrices, that is
	\begin{equation}\label{eq:matref3}
	\Pr\left(\norm{\frac{1}{m-s}\sum_{i=s+1}^{m}M_{\sigma(i)}-\frac{1}{m}\sum_{i=1}^{m}M_{i}}>\alpha\left(\frac{1}{\sqrt{\gamma (m-s)}}+\frac{1}{\gamma (m-s)}\right)\right) ~\leq~ 2d\exp\left(-\frac{\alpha}{2}\right).
	\end{equation}
		
	Now, notice that for any matrices $A,B,C$ and scalars $a,b$, it holds that
	\[
	\Pr(\norm{A-B}> a+b)~\leq~\Pr(\norm{A-C}> a)+\Pr(\norm{B-C}>b)
	\]
	(as the event $\norm{A-B}>a+b$ implies $\norm{A-C}+\norm{B-C}>a+b$). Using this observation and \eqref{eq:matref2}, \eqref{eq:matref3}, we have
	\begin{align*}
	&\Pr\left(\norm{\frac{1}{s}\sum_{i=1}^{s}M_{\sigma(i)}-\frac{1}{m-s}\sum_{i=s+1}^{m}M_{\sigma(i)}}~>~\frac{\alpha}{\sqrt{\gamma }}\left(\frac{1}{\sqrt{s}}+\frac{1}{\sqrt{m-s}}\right)+\frac{\alpha}{\gamma}\left(\frac{1}{s}+\frac{1}{m-s}\right)\right)\\
	&\leq 
	\Pr\left(\norm{\frac{1}{s}\sum_{i=1}^{s}M_{\sigma(i)}-\frac{1}{m}\sum_{i=1}^{m}M_{i}}~>~\alpha\left(\frac{1}{\sqrt{\gamma s}}+\frac{1}{\gamma s}\right)\right)\\
	&~~~~~+	\Pr\left(\norm{\frac{1}{m-s}\sum_{i=s+1}^{m}M_{\sigma(i)}-\frac{1}{m}\sum_{i=1}^{m}M_{i}}~>~\alpha\left(\frac{1}{\sqrt{\gamma(m-s) }}+\frac{1}{\gamma (m-s)}\right)\right)\\
	&\leq 4d\exp\left(-\frac{\alpha}{2}\right).
	\end{align*}
	The statement in the lemma now follows from a union bound argument over all possible $s=1,2,\ldots,m$.
\end{proof}

\begin{proposition}\label{prop:vvbound_squared1}
	Suppose each $f_i(\cdot)$ is of the form in \eqref{eq:squared}, where $\bx_i,\bw$ are in $\reals^d$, and $F(\cdot)$ is $\lambda$-strongly convex with $\lambda\in [1/m,1]$. Define
	\[
		\bv_{i}(t,s) = \nabla f_{i}(\bw_t)-\nabla f_{i}(\tilde{\bw}_s)+\nabla F(\tilde{\bw}_s).
	\]
	Then for any $t\leq m/2$ and any $\epsilon\in (0,1)$,
	\begin{align*}
	\E&\left[\inner{\bv_{\sigma(t)}(t,s),\bw_t-\bw^*}-\frac{1}{m}\sum_{i=1}^{m}\inner{\bv_i(t,s),\bw_t-\bw^*}\right]\\
	&~\leq~
	\frac{\epsilon}{2}+\frac{18}{\sqrt{\lambda m}}\log\left(\frac{64dmB^2}{\lambda\epsilon}\right)\cdot\E\left[F(\bw_t)+F(\tilde{\bw}_s)-2F(\bw^*)\right],
	\end{align*}
	where $d$ is the dimension.
\end{proposition}
\begin{proof}
	Define
	\[
	\bu_i(t,s) ~=~ \inner{\bv_i(t,s),\bw_t-\bw^*} ~=~ \inner{ \nabla f_i(\bw_t)-\nabla f_{i}(\tilde{\bw}_s)+\nabla F(\tilde{\bw}_s)~,~\bw_t-\bw^*},
	\]
	in which case the expectation in the proposition statement equals
	\[
	\E\left[\bu_{\sigma(t)}(t,s)-\frac{1}{m}\sum_{i=1}^{m}\bu_{i}(t,s)\right].
	\]
	Notice that $\bu_i(t,s)$ for all $i$ is independent of $\sigma(t),\ldots,\sigma(m)$ conditioned on $\sigma(1),\ldots,\sigma(t-1)$ (which determine $\bw_t$ and $\tilde{\bw}_s$). Therefore, we can apply \lemref{lem:key}, and get that the above equals
	\begin{equation}\label{eq:vvexp0}
	\frac{t-1}{m}\cdot \E\left[\bu_{t:m}(t,s)-\bu_{1:t-1}(t,s)\right].
	\end{equation}
	Recalling the definition of $\bu_i(t,s)$, and noting that the fixed $\inner{\nabla F(\tilde{\bw}_s),\bw_t-\bw^*}$ terms get cancelled out in the difference above, we get that \eqref{eq:vvexp0} equals
	\begin{equation}
	\frac{t-1}{m}\cdot \E\left[\check{\bu}_{t:m}(t,s)-\check{\bu}_{1:t-1}(t,s)\right].\label{eq:vvexp1}
	\end{equation}
	where
	\begin{align*}
	\check{\bu}_i(t,s) &= \inner{\nabla f_i(\bw_t)-\nabla f_i(\tilde{\bw}_s)~,~\bw_t-\bw^*}\\
	&= \inner{\bw_t-\tilde{\bw}_s,\bx_i}\cdot \inner{\bx_i,,\bw_t-\bw^*}+\hat{\lambda} \inner{\bw_t-\tilde{\bw}_s,\bw_t-\bw^*}.
	\end{align*}
	Again, the fixed $\hat{\lambda} \inner{\bw_t-\tilde{\bw}_s,\bw_t-\bw^*}$ terms get cancelled out in \eqref{eq:vvexp1}, so we can rewrite \eqref{eq:vvexp1} as
	\begin{equation}
	\frac{t-1}{m}\cdot \E\left[\breve{\bu}_{t:m}(t,s)-\breve{\bu}_{1:t-1}(t,s)\right]\label{eq:vvexp2}
	\end{equation}
	where
	\begin{align*}
	\breve{\bu}_i(t,s) &= 
	\inner{\bw_t-\tilde{\bw}_s,\bx_i}\cdot \inner{\bx_i,,\bw_t-\bw^*} \\
	&~=~ \inner{\bw_t-\bw^*,\bx_i}\cdot\inner{\bx_i,\bw_t-\bw^*}+\inner{\bw^*-\tilde{\bw}_s,\bx_i}\cdot\inner{\bx_i,\bw_t-\bw^*}\\
	&~=~ (\bw_t-\bw^*)^\top \bx_i \bx_i^\top (\bw_t-\bw^*)+(\bw^*-\tilde{\bw}_s)^\top\bx_i\bx_i^\top (\bw_t-\bw^*).
	\end{align*}
	Therefore, we can rewrite \eqref{eq:vvexp2} as
	\begin{align}
	&\frac{t-1}{m}\cdot \E\left[(\bw_t-\bw^*)^\top\left(\frac{1}{m-t+1}\sum_{i=t}^{m}\bx_{\sigma(i)}\bx_{\sigma(i)}^\top-\frac{1}{t-1}\sum_{i=1}^{t-1}\bx_{\sigma(i)}\bx_{\sigma(i)}^\top\right)(\bw_t-\bw^*)\right]\notag\\
	&+\frac{t-1}{m}\cdot \E\left[(\bw^*-\tilde{\bw}_s)^\top\left(\frac{1}{m-t+1}\sum_{i=t}^{m}\bx_{\sigma(i)}\bx_{\sigma(i)}^\top-\frac{1}{t-1}\sum_{i=1}^{t-1}\bx_{\sigma(i)}\bx_{\sigma(i)}^\top\right)(\bw_t-\bw^*)\right]\label{eq:vvexp3}
	\end{align}
	To continue, note that for any symmetric square matrix $M$, positive semidefinite matrix $A$, and vectors $\bw_1,\bw_2$, we have
	\begin{align*}
	|\bw_1^\top M \bw_2| &= \left(\bw_1^\top A \bw_1+\bw_2^\top A \bw_2\right)\left|\frac{\bw_1^\top M \bw_2}{\bw_1^\top A\bw_1+\bw_2^\top A \bw_2}\right|\\
	&\leq \left(\bw_1^\top A \bw_1+\bw_2^\top A \bw_2\right)\sup_{\bw_1,\bw_2}\left|\frac{\bw_1^\top M \bw_2}{\bw_1^\top A\bw_1+\bw_2^\top A \bw_2}\right|\\
	&\stackrel{(1)}{\leq} \left(\bw_1^\top A \bw_1+\bw_2^\top A \bw_2\right)\sup_{\bw_1,\bw_2}\left|\frac{\bw_1^\top A^{-1/2}MA^{-1/2} \bw_2}{\norm{\bw_1}^2+\norm{\bw_2}^2}\right|\\
	&\stackrel{(2)}{\leq} \left(\bw_1^\top A \bw_1+\bw_2^\top A \bw_2\right)\sup_{\bw_1,\bw_2}\left|\frac{\bw_1^\top A^{-1/2}MA^{-1/2} \bw_2}{2\norm{\bw_1}\norm{\bw_2}}\right|\\
	&\stackrel{(3)}{\leq} \frac{1}{2}\left(\bw_1^\top A \bw_1+\bw_2^\top A \bw_2\right)\norm{A^{-1/2}MA^{-1/2}},
	\end{align*}	
	where $(1)$ is by substituting $A^{-1/2}\bw_1, A^{-1/2}\bw_2$ in lieu of $\bw_1,\bw_2$ in the supremum, $(2)$ is by the identity $a^2+b^2 \geq 2ab$, and $(3)$ is by the fact that for any square matrix $X$, $|\bw_1^\top X \bw_2|\leq \norm{\bw_1}\norm{X}\norm{\bw_2}$. 
	Applying this inequality with
	\[ M=\frac{1}{m-t+1}\sum_{i=t}^{m}\bx_{\sigma(i)}\bx_{\sigma(i)}^\top-\frac{1}{t-1}\sum_{i=1}^{t-1}\bx_{\sigma(i)}\bx_{\sigma(i)}^\top~~,~~ A=\frac{1}{m}\sum_{i=1}^{m}\bx_{i}\bx_{i}^\top+\frac{\hat{\lambda}}{2}I,
	\] 
	$\bw_1$ being either $\bw^*-\tilde{\bw}_s$ or $\bw_t-\bw^*$, and $\bw_2 = \bw_t-\bw^*$, we can (somewhat loosely) upper bound \eqref{eq:vvexp3} by
	\[
	\frac{3(t-1)}{2m}\cdot\E\left[\left((\bw_t-\bw^*)^\top A (\bw_t-\bw^*)+(\tilde{\bw}_s-\bw^*)^\top A (\tilde{\bw}_s-\bw^*)\right)
	\norm{A^{-1/2}MA^{-1/2}}\right].
	\]
	Recalling that the objective function $F(\cdot)$ is actually of the form $F(\bw)=\bw^\top A \bw+\bb^\top \bw+c$ for the positive definite matrix $A$ as above, and some vector $\bb$ and scalar $c$, it is easily verified that $\bw^*= -\frac{1}{2}A^{-1}\bb$, and moreover, that
	\[
	(\bw-\bw^*)^\top A(\bw-\bw^*) = F(\bw)-F(\bw^*)
	\]
	for any $\bw$. Therefore, we can rewrite the above as
	\begin{align}
	&\frac{3(t-1)}{2m}\cdot\E\left[\left(F(\bw_t)+F(\tilde{\bw}_s)-2F(\bw^*)\right)
	\norm{A^{-1/2}MA^{-1/2}}\right]\notag\\
	&=\frac{3(t-1)}{2m}\cdot\E\left[\left(F(\bw_t)+F(\tilde{\bw}_s)-2F(\bw^*)\right)
	\norm{A^{-1/2}\left(\frac{\sum_{i=t}^{m}\bx_{\sigma(i)}\bx_{\sigma(i)}^\top}{m-t+1}-\frac{\sum_{i=1}^{t-1}\bx_{\sigma(i)}\bx_{\sigma(i)}^\top}{t-1}\right)A^{-1/2}}\right].\label{eq:vvexp5}
	\end{align}
	We now wish to use \lemref{lem:matserf}, which upper bounds the norm in the expression above with high probability. However, since the norm appears inside an expectation and is multiplied by another term, we need to proceed a bit more carefully. To that end, let $N$ denote the norm in the expression above, and let $D$ denote the expression $F(\bw_t)+F(\tilde{\bw}_s)-2F(\bw^*)$. We collect the following observations:
	\begin{itemize}
		\item The Hessian of the objective function $F(\cdot)$ is $\frac{1}{m}\sum_{i=1}^{m}\bx_i\bx_i^\top+\hat{\lambda}I$, whose minimal eigenvalue is at least $\lambda$ (since $F(\cdot)$ is assumed to be $\lambda$-strongly convex). Therefore, the minimal eigenvalue of $A$ as defined above is at least $\lambda/2$. Applying \lemref{lem:matserf}, $\Pr(N>\alpha \cdot q(t)) \leq 2dm\exp(-\alpha/2)$ for any $\alpha\geq 2$, where 
		\[
		q(t) = \sqrt{\frac{2}{\lambda}}\left(\frac{1}{\sqrt{t-1}}+\frac{1}{\sqrt{m-t+1}}\right)+\frac{2}{\lambda}\left(\frac{1}{t-1}+\frac{1}{m-t+1}\right).
		\]
		for any $t>1$, and $q(1)=0$.
		\item By assumption, $\norm{\bw_t}$, $\norm{\tilde{\bw}_s}$ and $\norm{\bw^*}$ are all at most $B$. Moreover, since the objective function $F(\cdot)$ is $1+\hat{\lambda} \leq 2$ smooth, $F(\bw_t)-F(\bw^*)\leq \norm{\bw_t-\bw^*}^2\le q 4B^2$ and $F(\tilde{\bw}_s)-F(\bw^*)\leq \norm{\tilde{\bw}_s-\bw^*}^2\leq 4B^2$. As a result, $D$ as defined above is in $[0,8B^2]$. 
		\item Since each $\bx_i \bx_i^\top$ has spectral norm at most $1$, $N$ is at most $\norm{A^{-1/2}}^2 \leq \frac{2}{\lambda}$.
	\end{itemize}
	Combining these observations, we have the following:
	\begin{align*}
	\E[DN] &= \Pr(N>\alpha q(t))\cdot\E[DN|N>\alpha q(t)]+\Pr(N\leq\alpha q(t))\cdot\E[DN|N\leq\alpha q(t)]\\
	&\leq 2dm\exp\left(-\frac{\alpha}{2}\right) \frac{16B^2}{\lambda}+\alpha q(t)\cdot\Pr(N\leq \alpha q(t))\cdot \E[D|N\leq \alpha q(t)]\\
	&\leq \frac{32dmB^2}{\lambda}\exp\left(-\frac{\alpha}{2}\right) +\alpha q(t)\cdot \E[D]
	\end{align*}
	for any $\alpha \geq 2$. In particular, picking $\alpha=2\log(64dmB^2/\lambda\epsilon)$ (where recall that $\epsilon\in (0,1)$ is an arbitrary parameter), we get
	\[
	\E[DN] \leq \frac{\epsilon}{2}+2\log\left(\frac{64dmB^2}{\lambda\epsilon}\right)q(t)\cdot \E[D]. 
	\]
	Plugging in the definition of $D$, we get the following upper bound on \eqref{eq:vvexp5}:
	\begin{equation}\label{eq:vvexp6}
	\frac{3(t-1)}{2m}\left(\frac{\epsilon}{2}+2\log\left(\frac{64dmB^2}{\lambda\epsilon}\right)q(t)\cdot\E\left[F(\bw_t)+F(\tilde{\bw}_s)-2F(\bw^*)\right]\right).
	\end{equation}
	Recalling the definition of $q(t)$ and the assumption $t\leq m/2$ (and using the convention $0/\sqrt{0}=0$), we have
	\begin{align*}
	\frac{3(t-1)}{2m}\cdot q(t)&=~
	\frac{3(t-1)}{2m}\left(\sqrt{\frac{2}{\lambda}}\left(\frac{1}{\sqrt{t-1}}+\frac{1}{\sqrt{m-t+1}}\right)+\frac{3}{\lambda}\left(\frac{1}{t-1}+\frac{1}{m-t+1}\right)\right)\\
	&=\frac{3}{\sqrt{2\lambda}}\left(\frac{\sqrt{t-1}}{m}+\frac{t-1}{m\sqrt{m-t+1}}\right)+\frac{3}{\lambda}\left(\frac{1}{m}+\frac{t-1}{m(m-t+1)}\right)\\
	&\leq\frac{3}{\sqrt{2\lambda}}\left(\frac{\sqrt{m/2}}{m}+\frac{m/2}{m\sqrt{m/2}}\right)+\frac{3}{\lambda}\left(\frac{1}{m}+\frac{m/2}{m(m/2)}\right)\\
	&=\frac{3}{\sqrt{2\lambda}}\left(\frac{1}{\sqrt{2m}}+\frac{1}{\sqrt{2m}}\right)+\frac{3}{\lambda}\left(\frac{1}{m}+\frac{1}{m}\right)\\
	&= \frac{3}{\sqrt{\lambda m}}+\frac{6}{\lambda m}~=~ \frac{3}{\sqrt{\lambda m}}\left(1+\frac{2}{\sqrt{\lambda m}}\right),
	\end{align*}
	which by the assumption $\lambda \geq 1/m$ (hence $\lambda m\geq 1$), is at most $9/\sqrt{\lambda m}$. Substituting this back into \eqref{eq:vvexp6} and loosely upper bounding, we get the upper bound
	\[
	\frac{\epsilon}{2}+\frac{18}{\sqrt{\lambda m}}\log\left(\frac{64dmB^2}{\lambda\epsilon}\right)\cdot\E\left[F(\bw_t)+F(\tilde{\bw}_s)-2F(\bw^*)\right],
	\]
	as required.
\end{proof}

\begin{proposition}\label{prop:vvbound_squared2}
	Let
	\[
	\bv_{i}(t,s) = \nabla f_{i}(\bw_t)-\nabla f_{i}(\tilde{\bw}_s)+\nabla F(\tilde{\bw}_s).
	\]	
	and suppose each $f_i(\cdot)$ is $\mu$-smooth. Then for any $t\leq m/2$,
	\[
	\E[\norm{\bv_{\sigma(t)}(t,s)}^2] ~\leq~6\mu(F(\bw_t)+F(\tilde{\bw}_s)-2F(\bw^*))
	\]
\end{proposition}
\begin{proof}
	Since $\nabla F(\bw^*)=\mathbf{0}$, we can rewrite $\bv_{\sigma(t)}(t,s)$ as
	\[
	g_{\sigma(t)}(\bw_t)-g_{\sigma(t)}(\tilde{\bw}_s)+\left(\nabla F(\tilde{\bw}_s)-\nabla F(\bw^*)\right),
	\]
	where 
	\[
	g_{\sigma(t)}(\bw) = \nabla f_{\sigma(t)}(\bw)-\nabla f_{\sigma(t)}(\bw^*).
	\]
	Using the fact that $(a+b+c)^2 \leq 3(a^2+b^2+c^2)$ for any $a,b,c$, we have
	\begin{align}
	&\frac{1}{3}\E\left[\norm{\bv_{\sigma(t)}(t,s)}^2\right] \\&~\leq~
	\E\left[\norm{g_{\sigma(t)}(\bw_t)}^2\right]
	+
	\E\left[\norm{g_{\sigma(t)}(\tilde{\bw}_s)}^2\right]
	+
	\E\left[\norm{\nabla F(\tilde{\bw}_s)-\nabla F(\bw^*)}^2\right]\notag\\
	&~=~\E\left[\frac{1}{m}\sum_{i=1}^{m}\norm{g_{i}(\bw_t)}^2\right]+
	\E\left[\frac{1}{m}\sum_{i=1}^{m}\norm{g_{i}(\tilde{\bw}_s)}^2\right]+\E\left[\norm{\nabla F(\tilde{\bw}_s)-\nabla F(\bw^*)}^2\right]\notag\\
	&~~~~~~+
	\E\left[\norm{g_{\sigma(t)}(\bw_t)}^2-\frac{1}{m}\sum_{i=1}^{m}\norm{g_{i}(\bw_t)}^2\right]+
	\E\left[\norm{g_{\sigma(t)}(\tilde{\bw}_s)}^2-\frac{1}{m}\sum_{i=1}^{m}\norm{g_{i}(\tilde{\bw}_s)}^2\right]\label{eq:v3}
	\end{align}
	We now rely on a simple technical result proven as part of the standard SVRG analysis (see equation (8) in \cite{johnson2013accelerating}), which states that if $P(\bw)=\frac{1}{n}\sum_{i=1}^{n}\psi_i(\bw)$, where each $\psi_i$ is convex and $\mu$-smooth, and $P$ is minimized at $\bw^*$, then for all $\bw$.
	\begin{equation}\label{eq:svrgb}
	\frac{1}{n}\sum_{i=1}^{n}\norm{\nabla\psi_i(\bw)-\nabla \psi_i(\bw^*)}^2\leq 2\mu\left(P(\bw)-P(\bw^*)\right)
	\end{equation}
	Applying this inequality on each of the first 3 terms in \eqref{eq:v3} (i.e. taking either $\psi_i(\cdot)=f_i(\cdot)$ and $n=m$, or $\psi(\cdot)=F(\cdot)$ and $n=1$), we get the upper bound
	\begin{align*}
	&2\mu(F(\bw_t)-F(\bw^*))+2\mu(F(\tilde{\bw}_s)-F(\bw^*))+2\mu(F(\tilde{\bw}_s)-F(\bw^*))\\
	&~~~~+
	\E\left[\norm{g_{\sigma(t)}(\bw_t)}^2-\frac{1}{m}\sum_{i=1}^{m}\norm{g_{i}(\bw_t)}^2\right]+
	\E\left[\norm{g_{\sigma(t)}(\tilde{\bw}_s)}^2-\frac{1}{m}\sum_{i=1}^{m}\norm{g_{i}(\tilde{\bw}_s)}^2\right]\label{eq:v32}\\
	&=~
	2\mu\left(F(\bw_t)-F(\bw^*)\right)+4\mu\left(F(\tilde{\bw}_s)-F(\bw^*)\right)
	\\
	&~~~~+~
	\E\left[\norm{g_{\sigma(t)}(\bw_t)}^2-\frac{1}{m}\sum_{i=1}^{m}\norm{g_{i}(\bw_t)}^2\right]+
	\E\left[\norm{g_{\sigma(t)}(\tilde{\bw}_s)}^2-\frac{1}{m}\sum_{i=1}^{m}\norm{g_{i}(\tilde{\bw}_s)}^2\right].
	\end{align*}
	Loosely upper bounding this and applying \lemref{lem:key}, we get the upper bound
	\begin{align*}
	&4\mu\left(F(\bw_t)+F(\tilde{\bw}_s)-2F(\bw^*)\right)
	~+~
	\frac{t-1}{m}\cdot \E\left[\frac{\sum_{i=t}^{m}\norm{g_{\sigma(i)}(\bw_t)}}{m-t+1}-\frac{\sum_{i=1}^{t-1}\norm{g_{\sigma(i)}(\bw_t)}^2}{t-1}\right]\\
	&~~~~+~
	\frac{t-1}{m}\cdot
	\E\left[\frac{\sum_{i=t}^{m}\norm{g_{\sigma(i)}(\tilde{\bw}_s)}}{m-t+1}-\frac{\sum_{i=1}^{t-1}\norm{g_{\sigma(i)}(\tilde{\bw}_s)}^2}{t-1}\right]\\
	&\leq~ 4\mu\left(F(\bw_t)+F(\tilde{\bw}_s)-2F(\bw^*)\right)+\frac{t-1}{m}\cdot\E\left[\frac{
		\sum_{i=t}^{m}\norm{g_{\sigma(i)}(\bw_t)}^2}{m-t+1}+\frac{\sum_{i=t}^{m}\norm{g_{\sigma(t)}(\tilde{\bw}_s)}^2}{m-t+1}\right]\\
	&\leq~ 4\mu\left(F(\bw_t)+F(\tilde{\bw}_s)-2F(\bw^*)\right)+\frac{t-1}{m}\cdot\E\left[\frac{
		\sum_{i=1}^{m}\norm{g_{\sigma(i)}(\bw_t)}^2}{m-t+1}+\frac{\sum_{i=1}^{m}\norm{g_{\sigma(i)}(\tilde{\bw}_s)}^2}{m-t+1}\right]\\
	&=~ 4\mu\left(F(\bw_t)+F(\tilde{\bw}_s)-2F(\bw^*)\right)+\frac{t-1}{m-t+1}\cdot\E\left[\frac{
		\sum_{i=1}^{m}\norm{g_{i}(\bw_t)}^2}{m}+\frac{\sum_{i=1}^{m}\norm{g_{i}(\tilde{\bw}_s)}^2}{m}\right].	
	\end{align*}
	Since we assume $t\leq m/2$, we have $\frac{t-1}{m-t+1}\leq \frac{m/2}{m/2} = 1$. Plugging this in, and applying \eqref{eq:svrgb} on the $\frac{1}{m}\sum_{i=1}^{m}\norm{g_{\sigma(i)}(\bw_t)}^2$ and $\frac{1}{m}\sum_{i=1}^{m}\norm{g_{\sigma(i)}(\tilde{\bw}_s)}^2$ terms, this is at most
	\begin{align*}
	&4\mu\left(F(\bw_t)+F(\tilde{\bw}_s)-2F(\bw^*)\right)+1\cdot\left(2\mu(F(\bw_t)-F(\bw^*))+2\mu(F(\tilde{\bw}_s)-F(\bw^*))\right)\\
	&= 6\mu(F(\bw_t)+F(\tilde{\bw}_s)-2F(\bw^*))
	\end{align*}
	as required.
\end{proof}
	
	\begin{proof}[Proof of \thmref{thm:svrg}]
		Consider some specific epoch $s$ and iteration $t$. We have
		\[
		\bw_{t+1} = \bw_t-\bv_{\sigma(t)}(t,s),
		\]
		where
		\[
		\bv_{i}(t,s) = \nabla f_{i}(\bw_t)-\nabla f_{i}(\tilde{\bw}_s)+\nabla F(\tilde{\bw}_s).
		\]
		Therefore,
		\begin{align}
		\E[\norm{\bw_{t+1}-\bw^*}^2] &=~ \E[\norm{\bw_t-\eta\bv_{\sigma(t)}(t,s)}^2]\notag\\
		&=~\E[\norm{\bw_t-\bw^*}^2]-2\eta\cdot \E[\inner{\bv_{\sigma(t)}(t,s),\bw_t-\bw^*}]+\eta^2\E[\norm{\bv_{\sigma(t)}(t,s)}^2].\label{eq:svrg0}
		\end{align}	
		Applying Proposition \ref{prop:vvbound_squared1} and Proposition \ref{prop:vvbound_squared2} (assuming that $t\leq m/2$, which we will verify later, and noting that $\lambda\geq 1/m$ by the assumptions on $\eta,T$ and $m$, and that each $f_i(\cdot)$ is $1+\hat{\lambda}\leq 2$-smooth), \eqref{eq:svrg0} is at most
		\begin{align*}
		&\E[\norm{\bw_t-\bw^*}^2]-2\eta\cdot \E\left[\frac{1}{m}\sum_{i=1}^{m}\inner{\bv_{i}(t,s),\bw_t-\bw^*}\right]+12\eta^2(F(\bw_t)+F(\tilde{\bw}_s)-2F(\bw^*))\\
		&~~~~+2\eta\left(\frac{\epsilon}{2}+\frac{18}{\sqrt{\lambda m}}\log\left(\frac{64dmB^2}{\lambda\epsilon}\right)\cdot\E\left[F(\bw_t)+F(\tilde{\bw}_s)-2F(\bw^*)\right]\right)\\
		&=~ \E[\norm{\bw_t-\bw^*}^2]-2\eta\cdot \E\left[\inner{\nabla F(\bw_t),\bw_t-\bw^*}\right]\\
		&~~~~+\eta\epsilon+2\eta \left(6\eta+\frac{18}{\sqrt{\lambda m}}\log\left(\frac{64dmB^2}{\lambda\epsilon}\right)\right)\cdot\E\left[F(\bw_t)+F(\tilde{\bw}_s)-2F(\bw^*)\right].
		\end{align*}
		Since $F(\cdot)$ is convex, $\inner{\nabla F(\bw_t),\bw_t-\bw^*}\geq F(\bw_t)-F(\bw^*)$, so we can upper bound the above by
		\begin{align*}
		&\E[\norm{\bw_t-\bw^*}^2] +\eta\epsilon+2\eta \left(6\eta+\frac{18}{\sqrt{\lambda m}}\log\left(\frac{64dmB^2}{\lambda\epsilon}\right)-1\right)\cdot\E\left[F(\bw_t)-F(\bw^*)\right]\\
		&~~~~+2\eta \left(6\eta+\frac{18}{\sqrt{\lambda m}}\log\left(\frac{64dmB^2}{\lambda\epsilon}\right)\right)\cdot\E\left[F(\tilde{\bw}_s)-F(\bw^*)\right].
		\end{align*}
		Recalling that this is an upper bound on $\E[\norm{\bw_{t+1}-\bw^*}^2]$ and changing sides, we get
		\begin{align*}
		&2\eta \left(1-6\eta-\frac{18}{\sqrt{\lambda m}}\log\left(\frac{64dmB^2}{\lambda\epsilon}\right)\right)\cdot\E\left[F(\bw_t)-F(\bw^*)\right]\\
		&~\leq~ \E[\norm{\bw_t-\bw^*}^2]-\E[\norm{\bw_{t+1}-\bw^*}^2]+\eta\epsilon\\
		&~~~~~~~+12\eta \left(\eta+\frac{3}{\sqrt{\lambda m}}\log\left(\frac{64dmB^2}{\lambda\epsilon}\right)\right)\cdot\E\left[F(\tilde{\bw}_s)-F(\bw^*)\right].
		\end{align*}
		Summing over all $t=(s-1)T+1,\ldots,sT$ in the epoch (recalling that the first one corresponds to $\tilde{\bw}_s$) and dividing by $\eta T$, we get
		\begin{align*}
		&2\left(1-6\eta-\frac{18}{\sqrt{\lambda m}}\log\left(\frac{64dmB^2}{\lambda\epsilon}\right)\right)\cdot\frac{1}{T}\sum_{t=(s-1)T+1}^{sT}\E\left[F(\bw_t)-F(\bw^*)\right]\\
		&~\leq~ \frac{1}{\eta T}\cdot\E[\norm{\tilde{\bw}_s-\bw^*}^2]+\epsilon\\
		&~~~~~~~+12 \left(\eta+\frac{3}{\sqrt{\lambda m}}\log\left(\frac{64dmB^2}{\lambda\epsilon}\right)\right)\cdot\E\left[F(\tilde{\bw}_s)-F(\bw^*)\right].
		\end{align*}	
		Since $F(\cdot)$ is $\lambda$-strongly convex, we have $\norm{\tilde{\bw}_s-\bw^*}^2\leq \frac{2}{\lambda}(F(\tilde{\bw}_s)-F(\bw^*))$. Plugging this in and simplifying a bit leads to
		\begin{align*}
		&\E\left[\frac{1}{T}\sum_{t=(s-1)T+1}^{sT} F(\bw_t)-F(\bw^*)\right]\\
		&~\leq~ \frac{\left(\frac{2}{\eta \lambda T}+12\left(\eta+\frac{3}{\sqrt{\lambda m}}\log\left(\frac{64dmB^2}{\lambda\epsilon}\right)\right)\right)\cdot \E\left[F(\tilde{\bw}_s)-F(\bw^*)\right]+\epsilon}{2 \left(1-6\eta-\frac{18}{\sqrt{\lambda m}}\log\left(\frac{64dmB^2}{\lambda\epsilon}\right)\right)}.
		\end{align*}
		The left hand side equals or upper bounds $\E[F(\tilde{\bw}_{s+1})]-F(\bw^*)$ (recall that we choose $\tilde{\bw}_{s+1}$ uniformly at random from all iterates produced in the epoch, or we take the average, in which case $\E[F(\tilde{\bw}_{s+1})-F(\bw^*)]$ is at most the left hand side by Jensen's inequality). As to the right hand side, if we assume
		\begin{equation}\label{eq:conditions}
		\eta = \frac{1}{c}~~,~~
		T\geq \frac{9}{\eta\lambda}~~,~~m\geq c\frac{\log^2(64dmB^2/\lambda\epsilon)}{\lambda}
		\end{equation}
		for a sufficiently large numerical constant $c$, we get that it is at most $\frac{1}{4}\cdot\E[F(\tilde{\bw}_s)-F(\bw^*)]+\frac{2}{3}\epsilon$. Therefore, we showed that
		\[
		\E\left[F(\tilde{\bw}_{s+1})-F(\bw^*)\right] ~\leq~ \frac{1}{4}\cdot\E[F(\tilde{\bw}_s)-F(\bw^*)]+\frac{2}{3}\epsilon.
		\]
		Unwinding this recursion, we get that after $s$ epochs,
		\begin{align*}
		\E\left[F(\tilde{\bw}_{s+1})-F(\bw^*)\right] &~\leq~ 4^{-s}\cdot\E[F(\tilde{\bw}_1)-F(\bw^*)]+\frac{2}{3}\epsilon\sum_{i=0}^{s-1}4^{-i}\\
		&~=~4^{-s}\cdot\E[F(\tilde{\bw}_1)-F(\bw^*)]+\frac{2}{3}\epsilon\cdot \frac{1-4^{-s}}{1-4^{-1}}\\
		&~<~4^{-s}\cdot\E[F(\tilde{\bw}_1)-F(\bw^*)]+\frac{8}{9}\epsilon.
		\end{align*}
		Since we assume that we start at the origin ($\tilde{\bw}_1=\mathbf{0}$), we have $F(\tilde{\bw}_1-F(\bw^*))\leq F(\mathbf{0}) = \frac{1}{m}\sum_{i=1}^{m}y_i^2 \leq 1$, so we get
		\[
		\E\left[F(\tilde{\bw}_{s+1})-F(\bw^*)\right] \leq 4^{-s}+\frac{8}{9}\epsilon.
		\]
		This is at most $\epsilon$ assuming $s\geq \log_4(9/\epsilon)$, so it is sufficient to have $\lceil \log_4(9/\epsilon)\rceil$ epochs to ensure suboptimality at most $\epsilon$ in expectation.
		
		Finally, note that we had $\lceil \log_4(9/\epsilon)\rceil$ epochs, in each of which we performed $T$ stochastic iterations. Therefore, the overall number of samples used is at most $\lceil \log_4(9/\epsilon)\rceil T$.
		Thus, to ensure the application of Propositions \ref{prop:vvbound_squared1} and  \ref{prop:vvbound_squared2} is valid, we need to ensure this is at most $m/2$, or that
		\[
		m~\geq~ 2\lceil \log_4(9/\epsilon)\rceil\cdot T.
		\]
		Combining this with \eqref{eq:conditions}, it is sufficient to require
		\[
		\eta = \frac{1}{c}~~,~~
		T\geq \frac{9}{\eta\lambda}~~,~~m\geq c\log^2\left(\frac{64dmB^2}{\lambda\epsilon}\right)T
		\]
		for any sufficiently large $c$.
	\end{proof}

\section{Summary and Open Questions}\label{sec:summary}

In this paper, we studied the convergence behavior of stochastic gradient methods using no-replacement sampling, which is widely used in practice but is theoretically poorly understood. We provided results in the no-replacement setting for various scenarios, including convex Lipschitz functions with any regret-minimizing algorithm; convex and strongly convex losses with respect to linear predictors, using stochastic gradient descent; and an analysis of SVRG in the case of regularized least squares. The latter results also has an application to distributed learning, yielding a nearly optimal algorithm (in terms of communication and computation) for regularized least squares on randomly partitioned data, under the mild assumption that the condition number is smaller than the data size per machine. 

Our work leaves several questions open. First, it would be interesting to remove the specific structural assumptions we have made on the loss functions, making only basic geometric assumptions such as strong convexity or smoothness. This would match the generality of the current theory on stochastic optimization using with-replacement sampling. Second, all our results strongly rely on concentration and uniform convergence, which are not needed when studying with-replacement stochastic gradient methods (at least when considering bounds which holds in expectation), and sometimes appears to lead to looser bounds, such as the $L,\mu$ factors in \thmref{thm:1t}. Thus, it would be interesting to have a more direct proof approach. Third, as discussed in the introduction, our theorems merely show that without-replacement sampling is not significantly worse than with-replacement sampling (in a worst-case sense) but this does not explain why without-replacement sampling often performs better. Although \cite{recht2012beneath,gurbuzbalaban2015random} provide some partial results in that direction, the full question  remains open. Finally, it would be interesting to extend our analysis of the SVRG algorithm to other fast-stochastic algorithms of a similar nature.

\subsubsection*{Acknowledgments}
This research is supported in part by an FP7 Marie Curie CIG grant, an Israel Science Foundation grant 425/13, and by the Intel Collaborative Research Institute for Computational Intelligence (ICRI-CI).

\bibliographystyle{plain}
\bibliography{mybib}

\newpage
\appendix

\section{Additional Technical Lemmas}\label{app:technical}

In this appendix, we collect a couple of purely technical lemmas used in certain parts of the paper.

\begin{lemma}\label{lem:tsqrtbound}
	If $T,m$ are positive integers, $T\leq m$, then
	\[
	\frac{1}{mT}\sum_{t=2}^{T}(t-1)\left(\sqrt{\frac{1}{t-1}}+\sqrt{\frac{1}{m-t+1}}\right) ~\leq~ \frac{2}{\sqrt{m}}.
	\]
\end{lemma}
\begin{proof}
	\begin{align*}
	&\frac{1}{mT}\sum_{t=2}^{T}(t-1)\left(\sqrt{\frac{1}{t-1}}+\sqrt{\frac{1}{m-t+1}}\right)\\
	&=\frac{1}{mT}\sum_{t=1}^{T-1}t\left(\sqrt{\frac{1}{t}}+\sqrt{\frac{1}{m-t}}\right)\\	
	&= \frac{1}{mT}\left(\sum_{t=1}^{T-1}\sqrt{t}+\sum_{t=1}^{T-1}
	\frac{t}{\sqrt{m-t}}\right).
	\end{align*}
	Since $\sqrt{t}$ and $t/\sqrt{m-t}$ are both increasing in $t$, we can upper bound the sums by integrals as follows:
	\begin{align*}
	&\leq
	\frac{1}{mT}\left(\int_{t=0}^{T}\sqrt{t}~dt+\int_{t=0}^{T}\frac{t}{\sqrt{m-t}}~dt\right)\\
	&=
	\frac{1}{mT}\left(\frac{2}{3}\cdot T^{3/2}+\left(-\frac{2}{3}\sqrt{m-t}\cdot(2m+t)\right)\Big|_{t=0}^{T}\right)\\
	&=
	\frac{1}{mT}\left(\frac{2}{3}\cdot T^{3/2}+\frac{2}{3}\left(2m\sqrt{m}-\sqrt{m-T}\cdot(2m+T)\right)\right)\\
	&=
	\frac{2}{3}\left(\frac{\sqrt{T}}{m}+2\left(\frac{\sqrt{m}}{T}-\sqrt{m-T}\cdot \left(\frac{1}{T}+\frac{1}{2m}\right)\right)\right)\\
	&=
	\frac{2}{3}\left(\frac{\sqrt{T}}{m}+\frac{2}{T}\left(\sqrt{m}-\sqrt{m-T}\cdot\left(1+\frac{T}{2m}\right)\right)\right)\\
	&\leq
	\frac{2}{3}\left(\frac{\sqrt{T}}{m}+\frac{2}{T}\left(\sqrt{m}-\sqrt{m-T}\right)\right)\\
	&=
	\frac{2}{3}\left(\frac{\sqrt{T}}{m}+\frac{2}{T}\left(\frac{T}{\sqrt{m}+\sqrt{m-T}}\right)\right)\\	
	&\leq
	\frac{2}{3}\left(\frac{\sqrt{T}}{m}+\frac{2}{\sqrt{m}}\right).	
	\end{align*}
	Since $T\leq m$, the above is at most $\frac{2}{3}\left(\frac{1}{\sqrt{m}}+\frac{2}{\sqrt{m}}\right) = \frac{2}{\sqrt{m}}$ as required.
\end{proof}

\begin{lemma}\label{lem:expprob}
	Let $X$ be a  random variable, which satisfies for any $\delta\in (0,1)$
	\[
	\Pr\left(X > a+b\log(1/\delta)\right) \leq \delta,
	\]
	where $a,b>0$. Then
	\[
	\E[X] ~\leq~ a+b.
	\]
	Furthermore, if $X$ is non-negative, then
	\[
	\sqrt{\E[X^2]} ~\leq~ \sqrt{2}\cdot(a+b).
	\]
\end{lemma}
\begin{proof}
	The condition in the lemma implies that for any $z\geq a$, 
	\[
	\Pr(X>z)\leq \exp\left(-\frac{z-a}{b}\right).
	\]
	Therefore,
	\begin{align*}
	\E[X] &~\leq~ \E[\max\{0,X\}] ~=~ \int_{z=0}^{\infty}\Pr(\max\{0,X\}\geq z)~dz ~\leq~ a+\int_{z=a}^{\infty}\Pr(\max\{0,X\}\geq z)~dz\\
	&~=~ a+\int_{z=a}^{\infty}\Pr(X\geq z)~dz~\leq~ a+\int_{z=a}^{\infty}\exp\left(-\frac{z-a}{b}\right)dz ~=~ a+\int_{z=0}^{\infty}\exp\left(-\frac{z}{b}\right)\\
	&= a+b.
	\end{align*}
	Similarly, if $X$ is non-negative,
	\begin{align*}
	\E[X^2] &~=~ \int_{z=0}^{\infty}\Pr(X^2\geq z)~dz ~\leq~ a^2+\int_{z=a^2}^{\infty}\Pr(X^2\geq z)~dz~=~a^2+\int_{z=a^2}^{\infty}\Pr(X\geq \sqrt{z})~dz\\
	&~\leq~ a^2+\int_{z=a^2}^{\infty}\exp\left(-\frac{\sqrt{z}-a}{b}\right)dz ~=~ a^2+2b(b+a)~\leq~ 2(a+b)^2,
	\end{align*}    	
	and the result follows by taking the square root.
\end{proof}

\section{Uniform Upper Bound on $F(\bw_t)-F(\bw^*)$ for SVRG}\label{app:B}

Below, we provide a crude bound on the parameter $B$ in \thmref{thm:svrg}, which upper bounds $F(\bw_t)-F(\bw^*)$ with probability $1$. Note that $B$ only appears inside log factors in the theorem, so it is enough that $\log(B)$ is reasonably small.

\begin{lemma}
	Suppose we run SVRG (algorithm \ref{alg:svrg}) with some parameter $T$ and step size $\eta\in (0,1)$ for $S$ epochs, where each $f_i(\cdot)$ is a regularized squared loss (as in \eqref{eq:squared}, with $\norm{\bx_i},|y_i|\leq 1$ for all $i$), and $F(\cdot)$ is $\lambda$-strongly convex with $\lambda\in (0,1)$. Then for every iterate $\bw_t$ produced by the algorithm, it holds with probability $1$ that
	\[
	\log(F(\bw_t)-F(\bw^*)) ~\leq~ 2S\cdot\log(5T)+ \log\left(\frac{4}{\lambda}\right).
	\]
\end{lemma}
Noting that \thmref{thm:svrg} requires only $S=\Ocal(\log(1/\epsilon))$ epochs with $T=\Theta(1/\lambda)$ stochastic iterations per epoch, we get that
\[
\log(F(\bw_t)-F(\bw^*)) = \Ocal\left(\log\left(\frac{1}{\epsilon}\right)\log(T)+\log\left(\frac{1}{\lambda}\right)\right)
\]
with probability $1$.

\begin{proof}
	Based on the SVRG update step, we have
	\begin{equation}\label{eq:bb1}
	\norm{\bw_{t+1}} \leq \norm{\bw_t-\eta \nabla f_{\sigma(t)}(\bw_t)}+\eta\norm{\nabla f_{\sigma(t)}(\tilde{\bw}_s)}+\eta\norm{\nabla F(\tilde{\bw}_s)}.
	\end{equation}
	Since we are considering the regularized squared loss, with $\norm{\bx_i}\leq 1,|y_i|\leq 1$ and $0\leq \hat{\lambda}\leq\lambda\leq 1$, the first term on the right hand side is
	\begin{align*}
	&\norm{\bw_t-\eta\left(\bx_{\sigma(t)}\bx_{\sigma(t)}^\top \bw_t-y_{\sigma(t)}\bx_{\sigma(t)}+\hat{\lambda} \bw_t\right)}~\leq~
	\norm{\left((1-\eta\hat{\lambda})I-\eta\cdot\bx_{\sigma(t)}\bx_{\sigma(t)}^\top\right)\bw_t}+\eta\norm{y_{\sigma(t)}\bx_{\sigma(t)}}\\
	&~~~~~~\leq~~\norm{\left((1-\eta\hat{\lambda})I-\eta\cdot\bx_{\sigma(t)}\bx_{\sigma(t)}^\top\right)}\norm{\bw_t}+\eta ~\leq~ \norm{\bw_t}+1,
	\end{align*}
	As to the second two terms on the right hand side of \eqref{eq:bb1}, we have for any $i$ by similar calculations that
	\[
	\norm{\nabla f_{i}(\bw)} = \norm{\left(\bx_i\bx_i^\top+\hat{\lambda} I\right)\bw-y_i\bx_i} ~\leq~ \left(1+\hat{\lambda}\right)\norm{\bw}+1 ~\leq~ 2\norm{\bw}+1
	\]
	as well as
	\[
	\norm{\nabla F(\bw)} \leq 2\norm{\bw}+1.
	\]
	Substituting these back into \eqref{eq:bb1} and loosely upper bounding, we get
	\[
	\norm{\bw_{t+1}} ~\leq~ \norm{\bw_t} + 4\left(\norm{\tilde{\bw}_s}+1\right).
	\]
	Recalling that each epoch is composed of $T$ such iterations, starting from $\tilde{\bw}_s$, and where $\tilde{\bw}_{s+1}$ is the average or a random draw from these $T$ iterations, we get that 
	\[
 \norm{\tilde{\bw}_{s+1}}~\leq~\norm{\tilde{\bw}_{s}}+4T\left(\norm{\tilde{\bw}_s}+1\right)~\leq~ 5T\left(\norm{\tilde{\bw}_s}+1\right),
	\]
	and moreover, the right hand side upper-bounds the norm of any iterate $\bw_t$ during that epoch. 
	Unrolling this inequality, and noting that $\norm{\tilde{\bw}_1}=0$, we get
	\[
	\norm{\tilde{\bw}_{S+1}} \leq (5T)^S\cdot 1 = (5T)^S,
	\]
	and $(5T)^S$ upper bounds the norm of any iterate $\bw_t$ during the algorithm's run.
	
	Turning to consider $\bw^*=\arg\min_{\bw}F(\bw)$, we must have $\norm{\bw^*}^2\leq 1/\lambda$ (to see why, note that any $\bw$ with squared norm larger than $1/\lambda$, $F(\bw)\geq F(\bw^*)+ \frac{\lambda}{2}\norm{\bw}^2> \frac{1}{2}$, yet $F(\mathbf{0})= \frac{1}{2m}\sum_{i=1}^{m}y_i^2 \leq \frac{1}{2}$, so $\bw$ cannot be an optimal solution). Moreover, $F(\cdot)$ is $2$-smooth, so for any iterate $\bw_t$,
	\[
	F(\bw_t)-F(\bw^*) ~\leq~ \norm{\bw_t-\bw^*}^2 ~\leq~ \left(\norm{\bw_t}+\norm{\bw^*}\right)^2 ~=~ \left((5T)^S+\sqrt{\frac{1}{\lambda}}\right)^2.
	\]
	Since $(5T)^S$ and $\sqrt{1/\lambda}$ are both at least $1$, this can be upper bounded by $\left(\frac{2(5T)^S}{\sqrt{\lambda}}\right)^2 = \frac{4}{\lambda}\cdot (5T)^{2S}$. Taking a logarithm, the result follows.
\end{proof}


\end{document}